\documentclass{article}

\usepackage[final]{neurips_2021}
\usepackage{bbm}
\usepackage{amssymb}
\usepackage{amsmath}
\usepackage{bm}
\usepackage{pict2e}
\usepackage{graphicx}
\usepackage{geometry}
\usepackage{amsfonts}
\usepackage{xcolor}
\usepackage[normalem]{ulem}
\usepackage{multicol}
\usepackage{caption}
\usepackage{subcaption}
\usepackage{mathtools}
\usepackage{subfiles}
\usepackage{amsthm}
\usepackage{diagbox}
\usepackage[toc,page,header]{appendix}
\usepackage{minitoc}

\usepackage{multirow}
\usepackage{pifont}
\usepackage{graphicx}
\usepackage{cancel}
\usepackage[most]{tcolorbox}
\tcbset{
    boxsep=0mm,
    boxrule=0pt,
    colframe=white,
    arc=0mm,
    left=0.5mm,
    right=0.5mm
}

\usepackage[utf8]{inputenc}
\usepackage[T1]{fontenc}
\usepackage{hyperref}
\hypersetup{
    colorlinks=true,
    linkcolor=blue,
    filecolor=magenta,      
    urlcolor=cyan,
    citecolor=blue,
}
\usepackage{url}
\usepackage{booktabs}
\usepackage{amsfonts}
\usepackage{nicefrac}
\usepackage{microtype}
\usepackage{xcolor}
\usepackage{stmaryrd}

\newtheorem*{theorem*}{Theorem}
\newtheorem{theorem}{Theorem}
\numberwithin{theorem}{section}
\newtheorem{lemma}[theorem]{Lemma}
\newtheorem*{lemma*}{Lemma}
\newtheorem{proposition}[theorem]{Proposition}

\renewcommand{\eqref}[1]{Equation~\ref{#1}}

\newcommand{\x}{{\mathbf x}}

\newcommand{\btheta}{{\boldsymbol \theta}}

\graphicspath{{figures/}}

\title{Precise characterization of the prior predictive distribution of deep ReLU networks}

\author{
  Lorenzo Noci$^*$\\ 
  Dept of Computer Science \\
  ETH Z\"urich\\
  \texttt{\small lorenzo.noci@inf.ethz.ch} \\
  \And
  Gregor Bachmann$^*$\\
  Dept of Computer Science \\
  ETH Z\"urich\\
  \texttt{\small gregor.bachmann@inf.ethz.ch} \\
  \And
  Kevin Roth$^*$\\
  Dept of Computer Science \\
  ETH Z\"urich\\
  \texttt{\small kevin.roth@inf.ethz.ch} \\
  \And
  Sebastian Nowozin\\
  Microsoft Research \\
  Cambridge, UK \\
  \texttt{\small Sebastian.Nowozin@microsoft.com} \\
  \And
  Thomas Hofmann \\
  Dept of Computer Science \\
  ETH Z\"urich\\
  \texttt{\small thomas.hofmann@inf.ethz.ch} \\   
}

\begin{document}

\maketitle

\doparttoc 
\faketableofcontents 

\begin{abstract}
Recent works on Bayesian neural networks (BNNs) have highlighted the need to better understand the implications of using Gaussian priors in combination with the compositional structure of the network architecture. 
Similar in spirit to the kind of analysis that has been developed to devise better initialization schemes for neural networks (cf. He- or Xavier initialization),
we derive a precise characterization of the prior predictive distribution of finite-width ReLU networks with Gaussian weights.
While theoretical results have been obtained for their heavy-tailedness,
the full characterization of the prior predictive distribution (i.e.\ its density, CDF and moments), remained unknown prior to this work.
Our analysis, based on the Meijer-G function, allows us to quantify the influence of architectural choices such as the width or depth of the network on the resulting shape of the prior predictive distribution. 
We also formally connect our results to previous work in the infinite width setting, demonstrating that the moments of the distribution converge to those of a normal log-normal mixture in the infinite depth limit. 
Finally, our results provide valuable guidance on prior design: 
for instance, controlling the predictive variance with depth- and width-informed priors on the weights of the network.
\end{abstract}


\section{Introduction}
It is well known that standard neural networks initialized with Gaussian weights tend to Gaussian processes \citep{rasmussen2003gaussian} in the infinite width limit \citep{neal1996priors, lee2018dnngp, matthews2018gaussiannetworks}, coined \textit{neural network Gaussian process} (NNGP) in the literature. Although the NNGP has been derived for a number of architectures, such as convolutional \citep{novak2019bayesiandeepcnngp, garrigaalonso2019convnetgp},  recurrent \citep{yang2019wide} and attention mechanisms \citep{hron2020infinite}, little is known about the finite width case.

The reason why the infinite width limit is relatively tractable to study is that uncorrelated but dependent units of intermediate layers become normally distributed due to the central limit theorem (CLT) and as a result, independent. In the finite width case, zero correlation does not imply independence,  rendering the analysis far more involved as we will outline in this paper.

One of the main motivations for our work is to better understand the implications of using Gaussian priors in combination with the compositional structure of the network architecture.
As argued by \cite{wilson2020bayesianperspective, wilson2020case}, the prior over parameters does not carry a meaningful interpretation; 
the prior that ultimately matters is the \textit{prior predictive distribution} that is induced when a prior over parameters is combined with a neural architecture \citep{wilson2020bayesianperspective, wilson2020case}.

Studying the properties of this prior predictive distribution is not an easy task, the main reason being the \textit{compositional} structure of a neural network, which ultimately boils down to products of random matrices with a given (non-linear) activation function. 
The main tools to study such products are the Mellin transform and the Meijer-G function \citep{meijer1936uber, springer1970distribution, mathai1993handbook, stojanac2017products}, both of which will be leveraged in this work to gain theoretical insights into the inner workings of BNN priors.

\paragraph{Contributions} 
Our results provide an important step towards understanding the interplay between architectural choices and the distributional properties of the prior predictive distribution, in particular:
\begin{itemize}
    \item We characterize the prior predictive density of finite-width ReLU networks of any depth through the framework of Meijer-G functions.
    \item We draw analytic insights about the shape of the distribution by studying its moments and the resulting heavy-tailedness. We disentangle the roles of width and depth, demonstrating how deeper networks become more and more heavy-tailed, while wider networks induce more Gaussian-like distributions.
    \item We connect our work to the infinite width setting by recovering and extending prior results \citep{lee2018dnngp, matthews2018gaussian} to the infinite depth limit. We describe the resulting distribution in terms of its moments and match it to a normal log-normal mixture \citep{lognormal}, empirically providing an excellent fit  even in the non-asymptotic regime. 
    \item Finally, we introduce generalized He priors, where a desired variance can be directly specified in the function space. This allows the practitioner to make an interpretable choice for the variance, instead of implicitly tuning it through the specification of each layer variance.
\end{itemize}

The rest of the paper is organized as follows: in Section \ref{sec:fully_connected_notation}, we introduce the relevant notation for the neural network that will be analyzed. We describe the prior works of \citet{lee2018dnngp, matthews2018gaussian} in more detail in Section \ref{nngp-section} . Then, in Section \ref{sec:toolbox}, we introduce the Meijer-G function and the necessary mathematical tools. In Section \ref{sec:pred_prior} we derive the probability density function for a linear network of any depth and extend these results to ReLU networks, which represents the key contribution of our work. In Section \ref{sec:analytic_insights}, we present several consequences of our analysis, including an extension of the infinite width setting to infinite depth as well as precise characterizations of the heavy-tailedness in the finite regime.
Finally, in Section \ref{sec:prior_design}, we show how one can design an architecture and a prior over the weights to achieve a desired prior predictive variance. 


\section{Related Work}
Although Bayesian inference in deep learning has recently risen in popularity, surprisingly little work has been devoted to investigating standard priors and their implied implicit biases in neural architectures. Only through the lens of infinite width, progress has been made \citep{neal1996priors}, establishing a Gaussian process behaviour at the output of the network. More recently, \citet{lee2018dnngp} and \citet{matthews2018gaussian} extended this result to arbitrary depth. We give a brief introduction in Section \ref{nngp-section}. Due to their appealing Gaussian process formulation, infinite width networks have been extensively studied theoretically, leading to novel insights into the training dynamics under gradient descent  \citep{jacot2018neural} and generalization \citep{arora2019finegrained}. 

Although theoretically very attractive, the usage of infinite width has been severely limited by its inferior empirical performance \citep{arora2019exact}.  While first insights into this gap have been obtained \citep{aitchison2020bigger, pmlr-v139-aitchison21a}, the picture is far from complete and a better understanding of finite width networks is still highly relevant for practical applications. 
In the finite regime, however, such precise characterizations in function space have been elusive so far and largely limited to empirical insights \citep{flam2017mapping, warpingbayes} 
and investigations of the heavy-tailedness of layers \citep{vladimirova2019understanding, fortuin2021bayesian}.  
The field of finite-width corrections has recently gained a lot of attraction. \cite{hanin2020products} studies the simultaneous limit of width and depth of the Jacobian of a ReLU net. More recently, a number of concurrent works appearing shortly before or after ours, such as  \cite{zavatone2021exact, roberts2021principles, li2021future} study properties of large-but-finite neural nets. Notably, \cite{zavatone2021exact} concurrently derived similar results on the characterization of the prior predictive distribution of finite width networks, however, we did not discover them until after we had completed this work. 
Note also that our novel limiting behaviours (cf.\ Sec.\ \ref{limiting_section}) and our analytical insights into heavy-tailedness (cf.\ Sec.\ \ref{kurtosis}) clearly distinguish our work from theirs. In addition, our results also offer valuable guidance on prior design for ML practitioners (cf.\ Sec.\  \ref{sec:prior_design}). \citet{li2021future} derive a similar limiting result in the infinite-width and depth setting in the case of Resnets (ours if for fully-connected, cf.\ Sec.\ \ref{limiting_section}), albeit with a completely different proof technique. Moreover, we also give precise insights into the prior predictive distribution for finite width (cf. \ Sec.\ \ref{sec:pred_prior}), whereas \citet{li2021future} only work with the limits.

Finally, this work is also related to the studies of signal propagation into finite-width random networks \citep{poole2016exponential, schoenholz2017deepinformationpropagation} and initialization \citep{he2015prelu, hanin2018start}. In particular, \cite{he2015prelu} uses a second moment analysis to specify the variance of the weights. In this sense, our approach extends it by deriving all the moments of the distribution. 

\section{Background}
\subsection{Fully Connected Neural Network}
\label{sec:fully_connected_notation}

Given an input $\bm{x} \in \mathbb{R}^d$, we define a $\btheta$-parameterized $L$ layer fully-connected neural network $f_{\bm{\theta}}(\bm{x})$ as the composition of layer-wise affine transformations $\bm{W}^{(l)} \in \mathbb{R}^{d_{l-1} \times d_{l}}$ and element-wise non-linearities $\sigma: \mathbb{R} \to \mathbb{R}$, 
\begin{equation}
    f_{\bm{\theta}}(\bm{x}) = \bm{W}^{(L)\top}\sigma\big(\bm{W}^{(L-1)\top} \dots \sigma\big(\bm{W}^{(1)\top}\bm{x}\big)\dots\big) \, ,
\end{equation}
where $\bm{\theta} = (\bm{W}^{(1)}, \dots , \bm{W}^{(L)})$ denotes the collection of all weights. 
Throughout this work, we assume standard initialization, i.e.\ $W_{ij}^{(l)} \sim \mathcal{N}(0,\sigma_l^2)$, where weights in each layer can have a different variance $\sigma_l^2$.
Often, it will be more convenient to work with the corresponding unit-level formulation, expressed through the recursive equations: 
\begin{align}
    & f^{(l)}_k(\bm{x}) = \sum_{j=1}^{d_{l-1}}W^{(l)}_{jk} g^{(l-1)}_j(\bm{x}) \,\,\ ,  \quad
     \bm{g}^{(l)}(\bm{x}) = \sigma(\bm{f}^{(l)}(\bm{x})) \, ,
\end{align}
When propagating an input $\bm{x}$ through the network we refer to $\bm{f}^{(l)}(\bm{x}) \in \mathbb{R}^{d_l}$ as the \textit{pre-activations} and to $\bm{g}^{(l)}(\bm{x}) \in \mathbb{R}^{d_l}$ as the \textit{post-activations}. When it is clear from the context, to enhance readability, we will occasionally abuse notation denoting $\bm{f}^{(l)} := \bm{f}^{(l)}(\bm{x})$ and $\bm{g}^{(l)} := \bm{g}^{(l)}(\bm{x})$. 
We will also use the short-hand $[m] = \{1, \dots, m\}$.
Imposing a probability distribution on the weights induces a distribution on the pre-activations $\bm{f}^{(l)}(\bm{x})$. Understanding the properties of this distribution is the main goal of this work.

\subsection{Prior Predictive Distribution and Infinite Width}
\label{nngp-section}
A precise characterization of the prior predictive distribution of a neural network has been established in the so-called infinite width setting \citep{neal1995bayesianneuralnetworks, lee2018dnngp, matthews2018gaussian}. By considering variances that scale inversely with the width, i.e.\ $\sigma_{l}^2 = \frac{1}{d_l}$ and fixed depth $L \in \mathbb{N}$, it can be shown that the implied prior predictive distribution converges in law to a Gaussian process:
\begin{equation}
    f_{\bm{\theta}} \xrightarrow[]{d} \mathcal{GP}(0, \Sigma^{(L)}) ,
\end{equation}

where $\Sigma^{(L)}: \mathbb{R}^{d} \times \mathbb{R}^{d}, \hspace{2mm} \bm{x}, \bm{x}' \mapsto \Sigma^{(L)}(\bm{x}, \bm{x}')$ is the NNGP kernel \citep{lee2018dnngp}, available in closed-form through the recursion 
\begin{align}
    &  \Sigma^{(1)}(\bm{x}, \bm{x}') = \bm{x}^{T}\bm{x}'\,\,\ ,  \quad
     \Sigma^{(l+1)}(\bm{x}, \bm{x'}) =  \mathbb{E}_{\bm{z} \sim \mathcal{N}(\bm{0},\tilde{\bm{\Sigma}}^{(l)})}\big[\sigma(z_1)\sigma(z_2)\big] \, ,
\end{align}
     
for $l=1, \dots, L-1$ and where $\tilde{\bm{\Sigma}}^{(l)} = \begin{pmatrix} \Sigma^{(l)}(\bm{x}, \bm{x}) & \Sigma^{(l)}(\bm{x}, \bm{x}') \\[1mm]
\Sigma^{(l)}(\bm{x}', \bm{x}) & \Sigma^{(l)}(\bm{x}', \bm{x}')\end{pmatrix} \in \mathbb{R}^{2 \times 2}$. The proof relies on the multivariate central limit theorem in conjunction with an inductive technique, letting hidden layer widths go to infinity in sequence. Due to the Gaussian nature of the limit, exact Bayesian inference becomes tractable and techniques from the Gaussian process literature can be readily applied \citep{rasmussen2003gaussian}. Although theoretically very appealing, from a practical point of view, infinite width networks are not as relevant due to their inferior performance 
\citep{novak2019bayesiandeepcnngp}. 
Gaining a better understanding of this performance gap is hence of utmost importance.

\subsection{Meijer-G function}
\label{sec:toolbox}
The Meijer-G function is the central tool to our analysis of the predictive prior distribution in the finite width regime. The Meijer-G function is a ubiquitous tool, appearing in a variety of scientific fields ranging from mathematical physics \citep{mathphysics} to symbolic integration software \citep{reduce} to electrical engineering \citep{electric}. Despite its high popularity in many technical fields, there have only been a handful of works in ML leveraging this elegant and convenient theoretical framework \citep{alaa2019demystifying, meijerg_learning}. 
In the following, we will introduce the Meijer-G function along with the relevant mathematical tools to develop our theory.

\begin{figure}
    \centering
    \begin{subfigure}[t]{0.45\textwidth}
        \includegraphics[width=\textwidth]{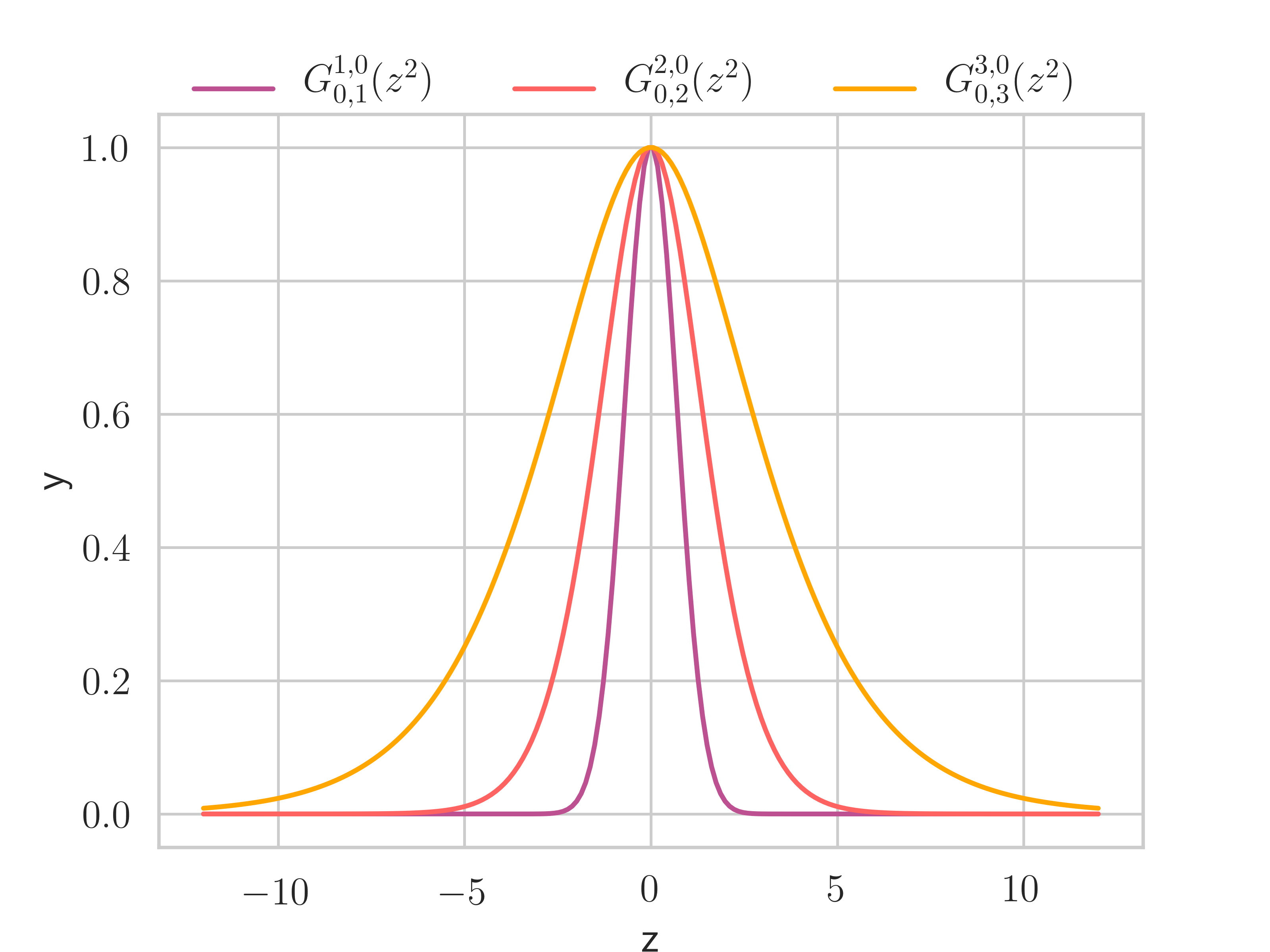}
        \caption{Plots of Meijer-G functions}
        \label{fig:meijer_g_plot}
    \end{subfigure}
    \begin{subfigure}[t]{0.38\textwidth}
        \includegraphics[width=\textwidth]{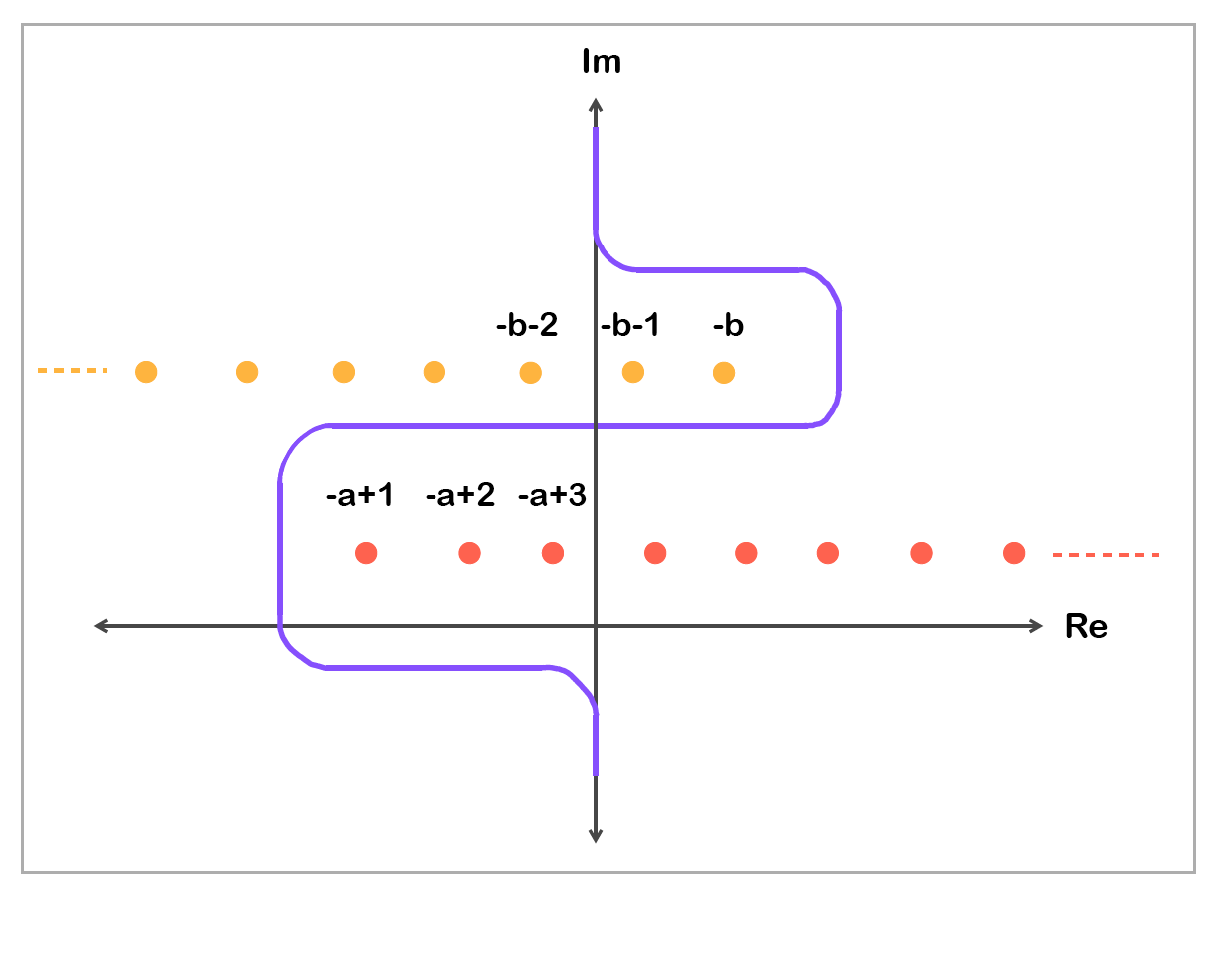}
        \caption{Visualization of integration path}
        \label{fig:meijer_g_path}
    \end{subfigure}
    \caption{(a) Plot of the Meijer-G functions $G^{l,0}_{0,l}\left(\cdot | \bm{b} \right)$ for $l=1, 2, 3$ and $\bm{b}=0, (0, 5), (0, 5, 5)$.
    (b)~An example path $\mathcal{L}$ in the complex plane. Notice how the orange singularities are always to the left of $\mathcal{L}$ and the red ones always to the right.}
    \label{meijer_g}
\end{figure}

For $s \in \mathbb{C}$ with real-part $\mathfrak{R}(s) >0$, denote by $\Gamma(s)$ the Gamma function defined as 
\begin{equation}
    \Gamma(s) = \int_{0}^{\infty}x^{s-1}e^{-x}dx .
\end{equation}

We then consider the analytic continuation of $\Gamma$, which extends it to the entire complex plane, as frequently done in complex analysis. Now fix $m,n,p,q \in \mathbb{N}$ such that $0 \leq m \leq q$ and $0 \leq n \leq p$ and consider $\bm{a} \in \mathbb{R}^{p}$, $\bm{b} \in \mathbb{R}^{q}$ such that $a_{i}-b_{j} \not \in \mathbb{Z}_{>0}$ $\forall i=1,\dots, p$ and $j=1,\dots, q$.
The Meijer-G function \citep{meijer1936uber, mathai1993handbook, mathai2006generalized}, is defined as:
\begin{equation}
    G^{m,n}_{p,q}\left(z | \begin{smallmatrix}\bm{a}\\ \bm{b} \end{smallmatrix} \right) = \frac{1}{2\pi i} \int_{\mathcal{L}} \chi(s) z^{-s} ds ,
\end{equation}
where:
\begin{equation}
    \chi(s) = \frac{\prod_{j=1}^{m}\Gamma(b_j + s) \prod_{j=1}^{n}\Gamma(1 - a_j - s)}{\prod_{j=m+1}^{q}\Gamma(1 - b_j - s) \prod_{j=n+1}^{p}\Gamma(a_j + s)} ,
\end{equation}
and the integration path $\mathcal{L}$ defines a suitable complex curve, described in the following. Recall that the function $\Gamma(s)$ has poles at $0, -1, -2, \dots$ all the way to $-\infty$. Hence, $\Gamma(1-a_j-s)$ has poles $-a_{j}+1, -a_{j}+2, \dots$, all the way to $\infty$, and $\Gamma(b_j + s)$ has poles $-b_j, -b_{j}-1, \dots$, all the way to $-\infty$. 
The path $\mathcal{L}$ is defined such that the poles of $\Gamma(b_j + s)$ are to the left of $\mathcal{L}$ and the ones of $\Gamma(1-a_j-s)$ are to the right of it. The condition $a_{i}-b_{j} \not \in \mathbb{Z}_{>0}$ makes sure we can find such a separation as this implies that the poles do not overlap. We illustrate an example of such a path in Figure \ref{fig:meijer_g_path}. 
In red we display the poles of $\Gamma(1-a_j+s)$  and in orange the poles of $\Gamma(b_j + s)$.
For interested readers we refer to \citet{Beals2013MeijerGA} for a more extensive introduction.

The defining property of Meijer-G functions is their closure under integration, i.e.\ the convolution of two Meijer-G functions is again a Meijer-G function.
Combined with the fact that most elementary functions can be written as a Meijer-G function, this property becomes extremely powerful at expressing complicated integrals neatly. 
Our proofs leverage this result extensively by expressing the integrands encountered in the prior predictive function as Meijer-G functions.  \\[2mm]
Throughout this text, we will only encounter Meijer-G functions of a simpler signature $G^{l,0}_{0,l}\left(\cdot | \bm{b} \right)$. For completeness, we show its functional form here:  
\begin{equation}
    G^{l,0}_{0,l}\left(z | \bm{b} \right) = \frac{1}{2\pi i} \int_{\mathcal{L}} z^{-s} \prod_{j=1}^{l}\Gamma(b_j + s)  ds .
\end{equation}
\noindent
Small values for $m \in \mathbb{N}$ correspond to familiar functions such as the exponential $G^{1,0}_{0,1}\left(z | 0 \right) = e^{-z}$ and the modified Bessel function of second kind $G^{2,0}_{0,2}\left(\frac{z^2}{4}\big{|}[\frac{\nu}{2}, -\frac{\nu}{2}] \right) = 2K_{\nu}(z)$. We visualize several Meijer-G functions of this form in Figure~\ref{fig:meijer_g_plot}. For illustrative purposes, we normalize the functions to have a maximum value of $1$.


\section{Predictive Priors for Neural Networks}
\label{sec:pred_prior}
In this section we detail our theoretical results on the predictive prior distribution implied by a fully-connected neural network with Gaussian weights, both with and without ReLU non-linearities. 
\paragraph{Linear Networks:}
First, we consider linear networks, i.e.\ fully-connected networks where the post-activations coincide with the pre-activations. They can be characterized as the product of Gaussian random matrices, for which the result in terms of Meijer-G functions is known (consider for instance \citet{ipsen2015products}). 
To highlight the differences between the linear and the non-linear approach, we re-prove the linear case leveraging our proof technique and notation. 
For simplicity, we assume w.l.o.g. that the input is normalized, i.e $|| \bm{x} {||} = 1$. We now present the resulting distribution of the predictive prior: 
\begin{tcolorbox}[boxrule=0pt, sharp corners]
\begin{theorem}
\label{thm:distr_linear_network_any_width}
Suppose $l \geq 1$, and the input has dimension $d_0$. Then, the joint marginal density of the random vector $\bm{f}^{(l)}$ (i.e.\ the density of the $l$-th layer pre-activations) is proportional to:
\begin{equation}
    p(\bm{f}^{(l)}) \propto G^{l, 0}_{0,l}\left(\frac{||\bm{f}^{(l)}||^2}{2^l  \sigma^2} \bigg \rvert  0, \frac{1}{2}\left(d_1 - d_l\right), \dots, \frac{1}{2}\left( d_{l-1} - d_l\right) \right) ,
\end{equation}
where $\sigma^2 = \prod_{i=1}^l \sigma_{i}^2$ .
\end{theorem}
\end{tcolorbox}
Our proof is based on an inductive technique by conditioning on the pre-activations $\bm{f}^{(l)}$ of the previous layer while analyzing the pre-activations of the next layer. Due to space constraints, we defer the full proof of Thm~\ref{thm:distr_linear_network_any_width} to the Appendix \ref{sec:proof_thm_linear_net_any_width}.  In the following, to give a flavor of our technique and highlight the great utility of the Meijer-G function, we present the base case as well as the key technical result (Lemma~\ref{lemma:integral_for_induction}) to perform the inductive step to obtain the final statement. 

\paragraph{Base case ($L=1$) :}
Here we restrict our attention to the first layer pre-activations which, given an input $\bm{x}$, are defined as:
\begin{equation}
    f^{(1)}_k(\bm{x}) = \sum_{j=1}^{d_0} W^{(1)}_{jk}\cdot x_j .
\end{equation}

 Conditioned on the input, this is a sum of $d_0$ i.i.d. Gaussian random variables {$W^{(1)}_{jk} \sim \mathcal{N}(0, \sigma^2_1)$}, which is again Gaussian with mean zero and variance $ \sigma_1^2 || \bm{x} ||^2$, i.e.\ $f^{(1)}_k(\bm{x}) \sim \mathcal{N}(0,  \sigma_1^2 || \bm{x} ||^2) \stackrel{(d)}{=}\mathcal{N}(0,  \sigma_1^2)$.
 where the last equality follows from $|| \x {||} = 1$.
 As we are conditioning on the input, the joint distribution of the first layer units is composed of $d_1$ independent Gaussians and hence
\begin{equation}
    \bm{f}^{(1)}(\bm{x}) \sim \mathcal{N}(0,  \sigma_1^2 \bm{I}) \; .
\end{equation}
Indeed, as anticipated from Thm.~\ref{thm:distr_linear_network_any_width}, the corresponding Meijer-G function encodes the Gaussian density:
\begin{equation}
    G^{1,0}_{0,1}\left(\frac{||\bm{f}^{(1)}||^2}{2 \sigma_1^2} \bigg \rvert 0 \right) = \exp\left(-\frac{||\bm{f}^{(1)}||^2}{2 \sigma_1^2} \right) .
\end{equation}

\paragraph{Induction Step:}

Conditioning on the previous pre-activations $\bm{f}^{(l-1)}$ brings us into a similar setting as in the base case because the resulting conditional distribution is again a Gaussian: 
\begin{equation}
    p(\bm{f}^{(l)}|\bm{f}^{(l-1)}) = \mathcal{N}(\bm{0}, \sigma_{l}^2 ||\bm{f}^{(l-1)}{||}_2^2) .
\end{equation}

In contrast to the fixed input $\bm{x}$, we now need to apply the law of total probability to integrate out the dependence on $\bm{f}^{(l-1)}$, leveraging the induction hypothesis for $p(\bm{f}^{(l-1)})$, to obtain the marginal distribution $p(\bm{f}^{(l)})$. This is where the Meijer-G function comes in handy as it can easily express such an integral. In combination with the closedness of the family under integration, this enables us to perform an inductive proof. We summarize this in the following Lemma:
\begin{tcolorbox}[boxrule=0pt, sharp corners]
\begin{lemma}
\label{lemma:integral_for_induction}
Let $\bm{f}^{l}$ and $\bm{f}^{l-1}$ be $d_l$-dimensional and $d_{l-1}$-dimensional vectors, respectively, where $l > 1$, $l \in \mathbb{N}$. Let $\sigma_{l}^2 > 0$, $\tilde{\sigma}^2 >0$ and $b_1, \dots b_{l-1} \in \mathbb{R}$. Then the following integral:
\begin{equation}
    I := \int_{\mathbb{R}^{d_{l-1}}} p(\bm{f}^{(l)}|\bm{f}^{(l-1)}) G_{0,l-1}^{l-1,0}\left(\frac{|| \bm{f}^{(l-1)}||^{2}}{2^{l-1}\tilde{\sigma}^{2} } \bigg\rvert b_1, \dots, b_{l-1}\right) d\bm{f}^{l-1} ,
\end{equation}
can be expressed as:
\begin{align}
     I = \frac{1}{C} G_{0,l}^{l,0}\left(\frac{|| \bm{f}^{(l)}||^2}{2^{l} \sigma^2 } \bigg\rvert 0, \frac{1}{2}(d_{l-1} - d_l) + b_1 , \dots, \frac{1}{2}(d_{l-1} - d_l) + b_{l-1}\right) ,
\end{align}
where $\sigma^2 = \sigma_{l}^2 \tilde{\sigma}^2$, and
$C \in \mathbb{R}$ is a constant available in closed-form.
\end{lemma}
\end{tcolorbox}
The proof largely relies on well-known integral identities involving the Meijer-G function. Again, due to its technical nature, we refer for the full proof and for the exact constant $C$ to the Appendix \ref{lemma:proof_integral_for_induction} and \ref{sec:normalization_constant}, respectively. 

\paragraph{ReLU Networks:}
\label{sec:prior_predictive_relu}
In the previous paragraph we have computed the prior predictive distribution for a linear network with Gaussian weights. Now, we extend these results to ReLU networks. The proof technique used is very similar to its linear counterpart, the main difference stems from the need to decompose the distribution over active and inactive ReLU cells. As a consequence, the resulting density is a superposition of different Meijer-G functions, each associated with a different active (linear) subnetwork. This is presented in the following:

\begin{tcolorbox}[boxrule=0pt, sharp corners]
\begin{theorem}
\label{THM:DISTR_RELU_NETWORK_ANY_WIDTH}
Suppose $l \geq 2$, and the input has dimension $d_0$. Define the multi-index set $\mathcal{R} = [d_1] \times \dots \times [d_{l-1}]$ and introduce the vector $\bm{u}^{\bm{r}} \in \mathbb{R}^{l-1}$ through its components $\bm{u}^{\bm{r}}_i = \frac{1}{2}(r_i-d_l)$.

\begin{equation}
    p(\bm{f}_{\text{ReLU}}^{(l)}) = \sum_{\bm{r} \in \mathcal{R}}q_{\bm{r}}G^{l, 0}_{0,l}\left(\frac{||\bm{f}_{\text{ReLU}}^{(l)}||^2}{2^l  \sigma^{2}} \bigg \rvert 0, \bm{u}^{\bm{r}} \right) + q_0 \delta(\bm{f}_{\text{ReLU}}^{(l)}) ,
\end{equation}
where $\sigma^2 = \prod_{i=1}^l\sigma_i^2$ and the individual weights are given by

\begin{equation}
    q_{\bm{r}} =\pi^{-\frac{d_l}{2}}2^{-\frac{l}{2}d_l} (\sigma^2)^{-\frac{d_l}{2}}\prod_{i=1}^{l-1}\binom{d_{i}}{r_i} \frac{1}{2^{d_{i}} \Gamma\left( \frac{r_i}{2}\right)}  ,
\end{equation}
and 
\begin{equation}
    q_0 = 1 - \prod_{i=1}^{l-1}\frac{2^{d_i}-1}{2^{d_i}}
\end{equation}
\end{theorem}
\end{tcolorbox}
We refer to Appendix \ref{sec:proof_thm_distr_relu_any_width} for the detailed proof. Observe that sub-networks with the same number of active units per layer induce the same density. The theorem reflects this symmetry by not summing over all possible sub-networks but only over the representatives of each equivalence class, absorbing the contribution of the other equivalent networks into $q_{\bm{r}}$. Indeed, notice that $|\mathcal{R}| = \prod_{i=1}^{l-1}d_i$, which is exactly the number of equivalent sub-networks contained in the original network. Notice also that $q_{\bm{r}}$ governs how much weight is assigned to the corresponding Meijer-G function. The normalization constant is incorporated into the weights $q_{\bm{r}}$ since the Meijer-G function does not integrate to $1$ (see Appendix \ref{lemma:linear_net_distr_norm_const}). As a result, the weights add up to the sum of the respective integration constants, and not to $1$.

\section{Analytic Insights into the Prior Predictive Distribution}
\label{sec:analytic_insights}
Here we highlight how one can use the mathematical machinery of Meijer-G functions to derive interesting insights,
relying on numerous mathematical results provided in the literature \citep{integrals1, integrals2, AndrewsLarryC2011Fgts}. With this rich line of work at our disposal, we can easily move from the rather abstract but mathematically convenient world of Meijer-G functions to very concrete results. We demonstrate this by recovering and extending the NNGP results provided in \citet{lee2018dnngp, matthews2018gaussian}. In particular, our analysis allows for simultaneous width and depth limits, showing how different limiting distributions emerge as a consequence of the growth with respect to $L$.  Finally, we characterize the heavy-tailedness of the prior-predictive for any width, providing further evidence that deeper models induce distributions with heavier tails, as observed in \citet{vladimirova2019understanding}. 
We hope that this work paves the way for further  progress in understanding priors, leveraging this novel connection through the rich literature on Meijer-G functions.

\subsection{Infinite Width: Recovering and Extending the NNGP for a Single Datapoint}
\label{limiting_section}
 We will start by giving an alternative proof in the linear case for the Gaussian behaviour emerging as the width of the network tends to infinity, recovering the results of \citet{lee2018dnngp,matthews2018gaussian} in the restricted setting of having just one fixed input $\bm{x}$. We extend their results in the following ways:
 \begin{itemize}
     \item[1.]  We provide a convergence proof  that is independent of the ordering of limits.
     \item[2.] We characterize the distributions arising from a simultaneous infinite width and depth limit, considering different growth rates for depth $L$.
 \end{itemize}  
 For ease of exposition, we focus on the equal width case, i.e.\ where $d_1 = \dots = d_{l-1} = m$ with one output $d_L = 1$. To have well-defined limits, one has to resort to the so-called NTK parametrization \citep{jacot2018neural}, which is achieved by setting the variances as $\sigma_1^2 = 1$ and $\sigma_i^2 = \frac{2}{m}$ \hspace{1mm} for $i=2, \dots, l$. We summarize the result in the following:  
\begin{tcolorbox}
\begin{theorem}
\label{relu_limit}
Consider the distribution of the output $p\left(f_{\text{ReLU}}^{(L)}\right)$, as defined in Thm.~\ref{THM:DISTR_RELU_NETWORK_ANY_WIDTH}. Denote  $X \sim \mathcal{N}(0,1)$, $Y \sim \mathcal{LN}(-\frac{5}{4}\gamma, \frac{5}{4}\gamma)$ for $X \perp Y$ and the resulting normal log-normal mixture by $Z=XY$, for $\gamma >0$. Let the depth grow as $L = c + \gamma m^{\beta}$ where $\beta \geq 0$ and $c \in \mathbb{N}$ fixed. Then it holds that 
\begin{equation}
    \mathbb{E}\left[\left(f^{(L)}_{\text{ReLU}}\right)^{2k}\right] \xrightarrow[]{m \xrightarrow[]{}\infty} \begin{cases} 
\mathbb{E}[X^{2k}] = (2k-1)!! \hspace{18mm} \text{if } \beta < 1\\[2mm]
\mathbb{E}[Z^{2k}] = e^{\frac{5}{2}\gamma k(k-1)}(2k-1)!! \hspace{3mm} \text{ if } \beta = 1\\[2mm]
\infty \hspace{44mm} \text{if } \beta>1
\end{cases} ,
\end{equation}
where $(2k-1)!! = (2k-1) \dots 3 \cdot 1$ denotes the double factorial (by symmetry, odd moments are zero). Moreover, for $\beta < 1$ it holds that 
\begin{equation}
    p(f^{(L)}_{\text{ReLU}}) \xrightarrow[]{d} X \hspace{4mm} \text{for } m \xrightarrow[]{} \infty ,
\end{equation}
for any $k > 1$.
\end{theorem}
\end{tcolorbox}
The proof leverages the fact that we can easily calculate all the moments of $f^{(L)}_{\text{ReLU}}$ at finite width using well-known integral identities involving the Meijer-G function. For the limit, care has to be taken as the moments of $f^{(L)}_{\text{ReLU}}$ involve the moments of the Binomial distribution $\operatorname{Bin}(m, \frac{1}{2})$ for which recursive but no analytic formulas are available. By carefully studying the leading coefficients of the resulting polynomial, we can establish the result in Thm.~\ref{relu_limit}. \\[2mm]
This result is in stark contrast to \citet{lee2018dnngp, matthews2018gaussian} which cannot deal with a simultaneous depth limit due to the inductive nature of their proof, but can only describe the special case of fixed depth ($\gamma=1, \beta=0$).
Notice that for depth growing asymptotically at a slower rate than width ($L=\gamma m^{\beta}$ for $\beta<1$), we obtain the convergence in distribution since the Gaussian distribution is identified by its moments (See Theorem 30.2 in \citet{Bill86}). For $\beta=1$, we can also prove convergence of the moments and identify them as arising from the normal log-normal mixture $Z$ \citep{lognormal}. Unfortunately, this does not suffice to conclude convergence in distribution as there are no known results on the identifiability of $Z$. Empirical evidence however, presented in Figure~\ref{fig:logcdf_relu} and Figure~\ref{fig:pdf_relu}, suggests that the normal log-normal mixture captures the distribution to an excellent degree of fidelity. We demonstrate this through the following experiment. Note first that $Z \xrightarrow[]{} \mathcal{N}(0,1)$ in law as $\gamma \xrightarrow[]{} 0$, which is expected, as for $\gamma \xrightarrow[]{} 0$, we assume a fixed depth $L=c$. We can hence use this result in the non-asymptotic regime by setting $\gamma = \frac{L}{m}$, expecting to interpolate between the two distributions as we vary depth and width. We illustrate the empirical findings in Figure~\ref{fig:logcdf_relu} and Figure~\ref{fig:pdf_relu}. We find an astonishing match to the true distribution both in terms of CDF and PDF. Moreover, as expected, we recover the Gaussian distribution for small values of $\gamma$.  \\[2mm]
The decoupling of the moments into two separate factors gives insights into the role of width and depth on the shape of the distribution. The emergence of the log-normal factor in this infinite depth limit highlights how deeper networks encourage more heavy-tailed distribution, if not countered by a sufficient amount of width. This becomes even more drastic once depth outgrows width ($\beta>1$), leading to divergence of all even moments. In the following section we will show that also in the finite width regime, heavy-tailedness becomes more pronounced as we increase the depth. 

\begin{figure}[ht!]
    \centering
    \begin{subfigure}[ht]{0.3\textwidth}
        \includegraphics[width=\textwidth]{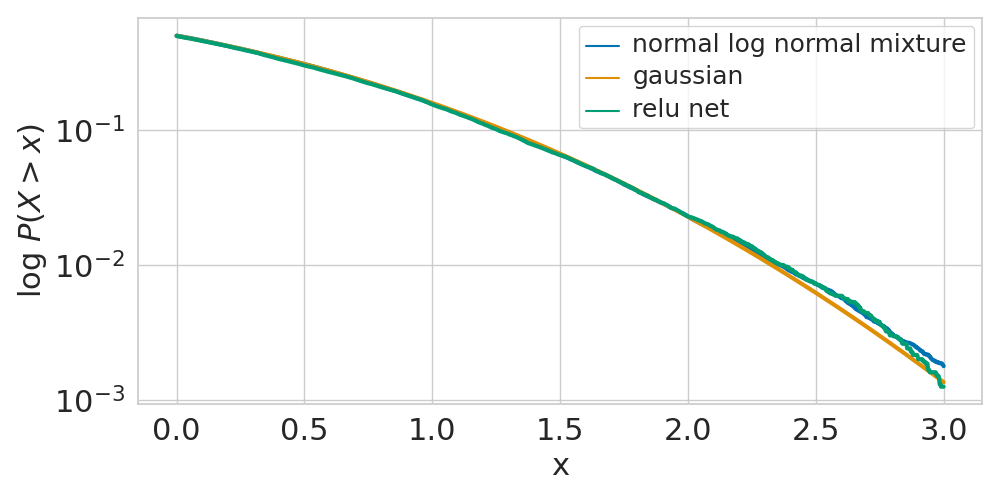}
        \caption{$\gamma = 0.01$}
        \label{fig:gamma_001_relu_cdf}
    \end{subfigure}
    \begin{subfigure}[ht]{0.3\textwidth}
        \includegraphics[width=\textwidth]{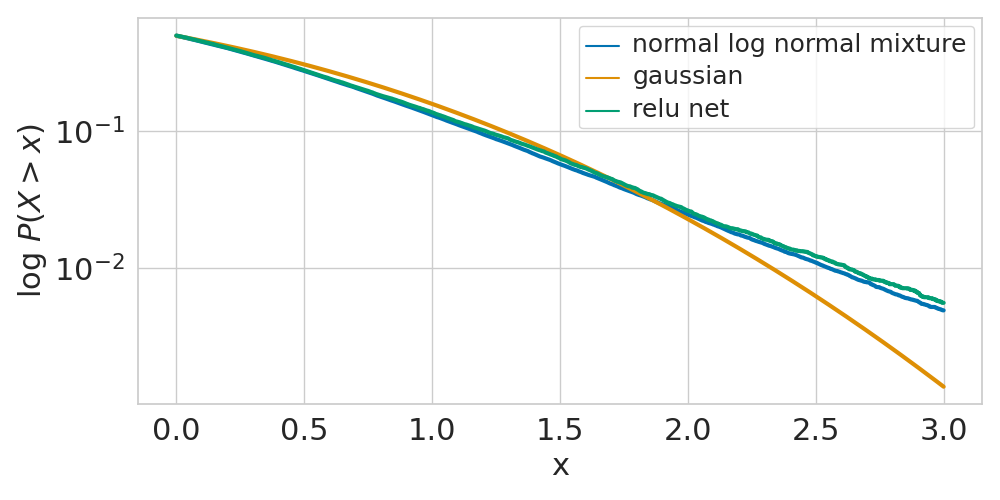}
        \caption{$\gamma = 0.1$}
        \label{fig:gamma_01_relu_cdf}
    \end{subfigure}
    \begin{subfigure}[ht]{0.3\textwidth}
        \includegraphics[width=\textwidth]{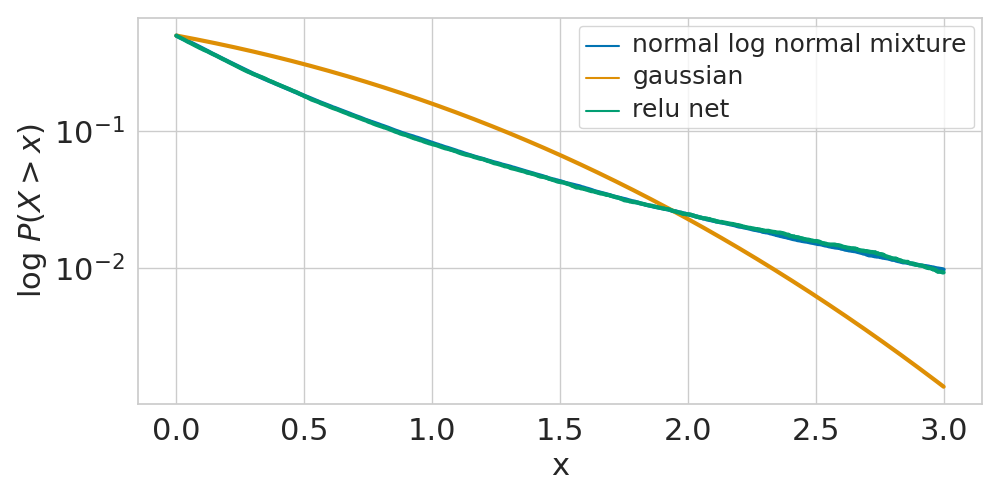}
        \caption{$\gamma = 0.5$}
        \label{fig:gamma_05_cdf}
    \end{subfigure}
    \caption{\textbf{Convergence to normal log-normal mixture:} We display the CDFs of a ReLU network, standard Gaussian CDF and the normal log-normal mixture. The width is fixed to $100$  while depth is increased from $1$ to $50$ using the parameter $\gamma$. The neural network CDF is constructed empirically by drawing $10^{4}$ samples from the prior. Note that as depth increases, the CDF departs from the Gaussian but still follows the normal log-normal mixture.}
    \label{fig:logcdf_relu}
\end{figure}

\begin{figure}[ht!]
    \centering
    \begin{subfigure}[ht]{0.3\textwidth}
        \includegraphics[width=\textwidth]{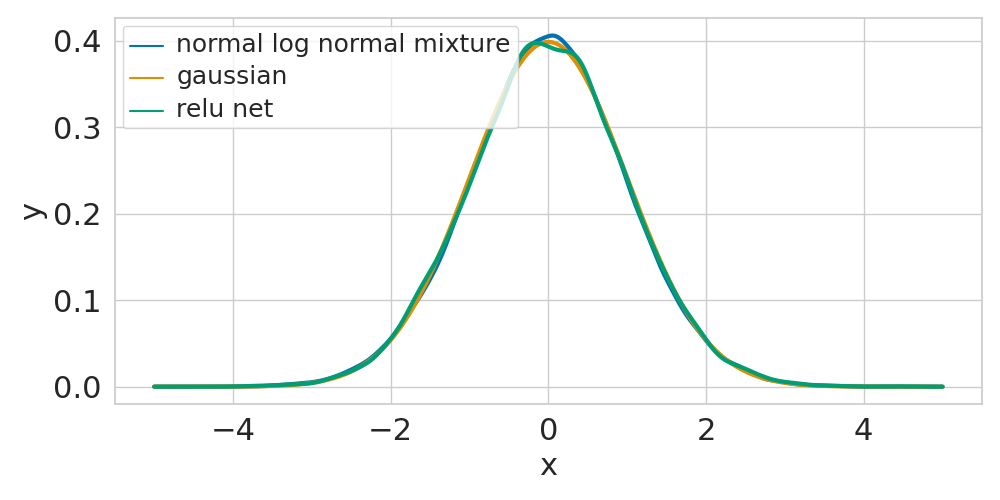}
        \caption{$\gamma = 0.01$}
        \label{fig:gamma_001_relu_pdf}
    \end{subfigure}
    \begin{subfigure}[ht]{0.3\textwidth}
        \includegraphics[width=\textwidth]{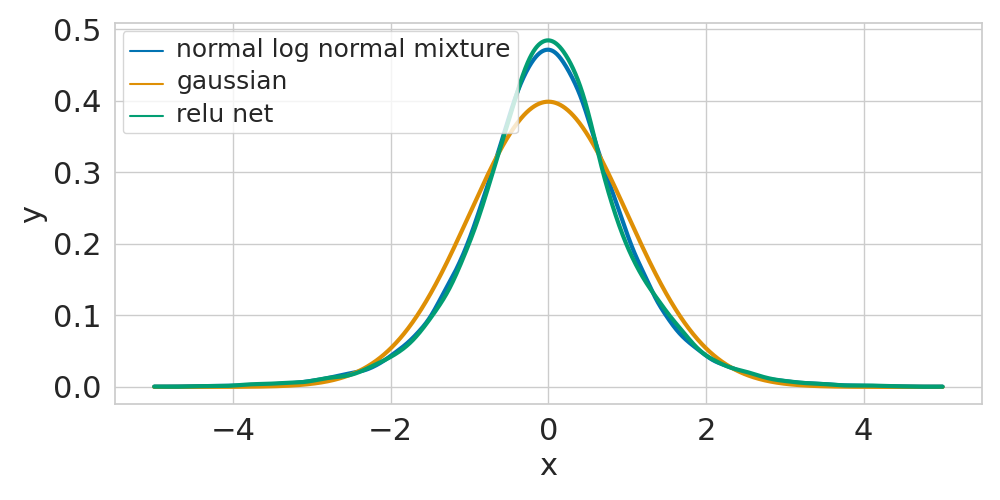}
        \caption{$\gamma = 0.1$}
        \label{fig:gamma_01_relu_pdf}
    \end{subfigure}
    \begin{subfigure}[ht]{0.3\textwidth}
        \includegraphics[width=\textwidth]{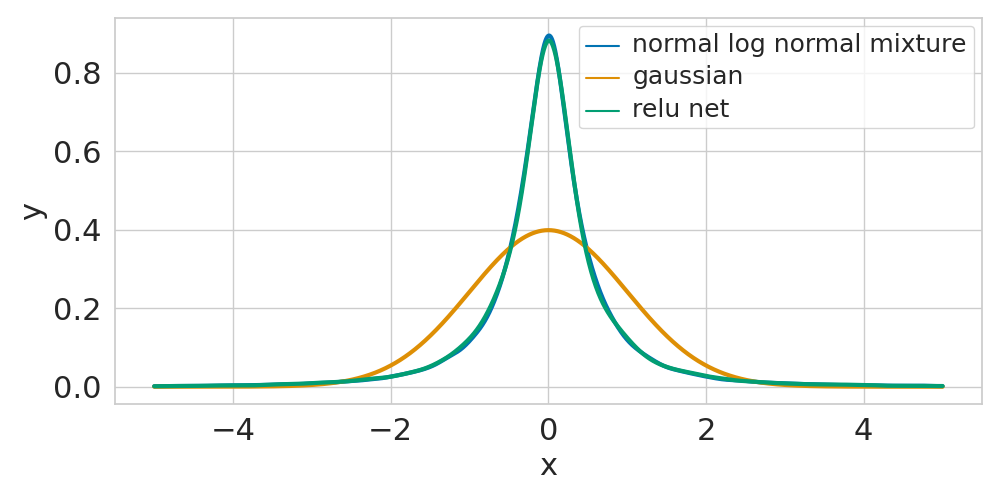}
        \caption{$\gamma = 0.5$}
        \label{fig:gamma_05_pdf}
    \end{subfigure}
    \caption{\textbf{Convergence to normal log-normal mixture:} We show the density plots for the same setting as in Figure \ref{fig:logcdf_relu}. Note that as depth increases, the density departs from the Gaussian but still follows the normal log-normal mixture.}
    \label{fig:pdf_relu}
\end{figure}

\subsection{Heavy-tailedness Increases with Depth}
\label{kurtosis}
From the moments analysis, it is simple to recover a known fact about the prior distribution of neural networks, namely that deeper layers are increasingly heavy-tailed \citep{vladimirova2019understanding}. To see this, we can derive the kurtosis, a standard measure of tailedness of the distribution \citep{westfall2014kurtosis}, defined as:
\begin{equation}
    \kappa := \frac{\mathbb{E}\left[\left(X - \mathbb{E}[X] \right)^{4}\right]}{\mathbb{V}[X]^2},
\end{equation}
where $X$ is a univariate random variable with finite fourth moment.
We can calculate the kurtosis of a ReLU network analytically at any width, relying on closed-form results for the lower order moments of the Binomial distribution. We outline the exact calculation in the Appendix \ref{relu_proofs} and state the resulting expression here:
\begin{equation}
    \kappa_{\text{ReLU}}(m,L) = 3 \left(\frac{m+5}{m}\right)^{L-1} .
\end{equation}

Note how the kurtosis increases with depth $L$ and decreases with the width $m$, highlighting once again the opposite roles those two parameters take regarding the shape of the prior.
As expected, for fixed depth $L$, the kurtosis converges to $3$, i.e.\ $\lim_{m \xrightarrow[]{}\infty} \kappa_{\text{ReLU}}(m,L) =3$, which is the kurtosis of a standard Gaussian variable $\mathcal{N}(0,1)$. For the simultaneous limit $L=\gamma m$, we find a value of $3e^{5\gamma}$, which exactly matches with the expression derived in Thm.~\ref{relu_limit}:
\begin{equation}
    \lim_{m \xrightarrow[]{}\infty}\kappa_{\text{ReLU}}(m, \gamma m) = \frac{\mathbb{E}[Z^4]}{\mathbb{E}[Z^2]^2} = \frac{e^{\frac{5}{2}\gamma  2}3!!}{e^{0}1!!} = 3e^{5\gamma} .
\end{equation}


\section{Prior Design}
\label{sec:prior_design}
In Section \ref{kurtosis}, we have already outlined how the choice of architecture influences the heavy-tailedness of the distribution. If a Gaussian-like output is desired, the architecture should be designed in such a way that the width $m$ significantly exceeds the depth $L$ (small $\gamma)$, while heavy-tailed predictive priors can be achieved by considering regimes where $L$ exceeds $m$ (big $\gamma)$. 
As we will see shortly after in this section, another consequence of our analysis is that we can now directly work with the variance in  function space, instead of implicitly tuning it by changing the variances at each layer. In this way, the variance over the weights has a clear interpretation in terms of predictive variance, making it easier for the deep learning practitioner to take an informed decision when designing a prior for BNNs.
This type of variance analysis has previously been used to devise  better initialization schemes for neural networks coined "He-initialization" \citep{he2015prelu} in the Gaussian case and "Xavier initialization" \citep{pmlr-v9-glorot10a} for uniform weight initializations. 
Therefore, we coin the resulting prior \textit{Generalized He-prior}.

\paragraph{Predictive variance:} We consider the parametrization $\sigma_1^2 = 1$ and $\sigma_i^2 = \frac{2t_i^2}{m}$, which, as we show in Appendix \ref{relu_proofs}, leads to the following predictive variance:
\begin{equation}
    \mathbb{V}[Z] = \sigma_{1}^2 \prod_{i=2}^{l}t_i^2 .
\end{equation}
 Suppose we want a desired output variance $\sigma^2$. This can be achieved as follows: let $a_i$, with $i = 2, \dots, l$, be $l-1$ coefficients such that $\sum_{i=1}^{l-1}a_i = l-1$. Then choose $t_i^2 = \left(\sigma^2\right)^{\frac{a_i}{l-1}}$, implying that the layer variances $\sigma_i^2$ are given as
\begin{equation}
    \sigma_{i}^2 = \frac{2}{m}\left(\sigma^2\right)^{\frac{a_i}{l-1}} ,
\end{equation}
He-priors correspond to the special case where $\sigma^2 = 1$ and $a_i = 1$, $i=1, \dots, l-1$, while a standard Gaussian $N(0, 1)$ prior results in $\sigma^2 = \left(\frac{m}{2}\right)^{l-1}$, $\sigma_1^2 = 1$, and $a_i = 1$, $i=1, \dots, l-1$. Note how "far" He-priors can be from standard Gaussian for relatively deep and wide neural networks. While rather innocent-looking, using a standard Gaussian $\mathcal{N}(0,1)$ prior can lead to an extremely high output variance.  

Combining with the previous insights, practitioners can now choose a desired output variance along with a desired level of heavy-tailedness in a controlled manner by specifying the architecture, i.e.\ the width $m$ and the depth $L$.

\section{Discussion}
Our work sheds light on the shape of the prior predictive distribution arising from imposing a Gaussian distribution on the weights. Leveraging the machinery of Meijer-G functions, we characterized the density of the output in the finite width regime and derived analytic insights into its properties such as moments and heavy-tailedness. An extension to the stochastic process setting in the spirit of \citet{lee2018dnngp, matthews2018gaussian} as well as to convolutional architectures \citep{novak2019bayesiandeepcnngp} could bring theory even closer to practice and we expect similar results to also hold in those cases. This is however beyond the scope of our work and we leave it as future work. 

Our technique enabled us to extend the NNGP framework to infinite depth, discovering how in the more general case, the resulting distribution shares the same moments as a normal log-normal distribution. This allowed us to disentangle the roles of width and depth, where the former induces a Gaussian-like distribution while the latter encourages heavier tails.  Empirically, we found that the normal log-normal mixture provides an excellent fit to the true distribution even in the non-asymptotic setting, capturing both the cumulative and probability density function to a very high degree of accuracy. This surprising observation begs further theoretical and empirical investigations. In particular, discovering a suitable stochastic process incorporating the normal log-normal mixture, could lend further insights into the inner workings of neural networks. Moreover, the role of heavy-tailedness regarding generalization is very intriguing, potentially being an important reason underlying the gap between infinite width networks and their finite counterparts. This is also in-line with recent empirical works on priors, suggesting that heavy-tailed distributions can increase the performance significantly \citep{fortuin2021bayesian}. 

Using these insights, we described how one can choose a fixed prior variance directly in the output space along with a desired level of heavy-tailedness resulting from the choice of the architecture. We leave it as future work to also consider higher moments to give more nuanced control over the resulting prior in the function space.
Finally, we hope that the introduction of the Meijer-G function sparks more theoretical research on BNN priors and their implied inductive biases in function space.

\bibliographystyle{apalike}
\bibliography{00_main}

\newpage


\appendix
\addcontentsline{toc}{section}{Appendix}
\part{Appendix} 
\parttoc 


\section{Properties of Meijer-G function}

Here we describe the properties of Meijer-G functions which we will use extensively in the following.

The first result concerns the Mellin transform of the Meijer-G function, which will be the key to solve the integrals that we will face later. 
\label{sec:prop_meijer_g}
\begin{proposition}[ Mellin transform of the Meijer G function]
\label{prop:mellin_transf_g}
\begin{equation}
    \int_0^{\infty} x^{s-1} G^{m,n}_{p,q}(wx) = w^{-s} \frac{\prod_{j=1}^{m}\Gamma(b_j + s) \prod_{j=1}^{n}\Gamma(1 - a_j - s)}{\prod_{j=m+1}^{q}\Gamma(1 - b_j - s) \prod_{j=n+1}^{p}\Gamma(a_j + s)} .
\end{equation}
\end{proposition}
\begin{proof}
See Chapter 3.2 and 2.3 of \cite{mathai2006generalized}. Conditions of validity: for the class of Meijer-G functions that we consider here ($p$, $n = 0$, $m=q=l$ and the coefficients are all real) the Mellin transform exists (see \cite{mathai2006generalized}, Section 2.3.1). 
\end{proof}
To establish the base case $m=1$, we need the following results. 
\begin{proposition}
\label{prop:meijer_g_identities}
The following identities hold:
\begin{itemize}
    \item $\exp(z) = G^{1,0}_{0,1}\left(-z \bigg \rvert 0 \right)$ .
    \item multiplication by power property: $z^{d}  G^{m,n}_{p,q}\left(z | \begin{smallmatrix}a_1 & \dots & a_p \\ b_1 & \dots & b_q \end{smallmatrix} \right) =  G^{m,n}_{p,q}\left(z | \begin{smallmatrix}a_1 + d & \dots & a_p + d \\ b_1 + d & \dots & b_q + d \end{smallmatrix} \right)$ . 
\end{itemize}
\begin{proof}
See Chapter 2.6 of \cite{mathai2006generalized} for the first identity. The last property follows directly from the definition of Meijer-G function.
\end{proof}
\end{proposition}
To perform the inductive step, we will encounter the following integral, that can be expressed in terms of the Meijer-G function.
\begin{proposition}
\begin{equation}
\label{eq:meijer_g_property_for_cdf}
    \int_1^{\infty} x^{-\rho}(x-1)^{\sigma - 1}G^{m,n}_{p,q}\left(\alpha x | \begin{smallmatrix}a_1 & \dots & a_p \\ b_1 & \dots & b_q \end{smallmatrix} \right) dx = \Gamma(\sigma) G^{m+1,n}_{p+1,q+1}\left(\alpha | \begin{smallmatrix}a_1 & \dots & a_p & \rho \\ \rho - \sigma & b_1 & \dots & b_q \end{smallmatrix} \right) .
\end{equation}
\end{proposition}
The conditions of validity for the class of Meijer-G that we consider here are again satisfied (see Appendix B of \cite{stojanac2017products}). 

\section{Proof for Linear Networks and  Derivation of their Moments}
 
Here we collect all the results regarding linear networks, establishing the relevant technical Lemmas to derive the density and calculate the moments of the resulting distribution.
\label{app:linear_proofs}
\begin{lemma}
The units of any layer are uncorrelated, i.e.\\ 
\begin{equation}
    Cov\left(f_k^{(l)}, f_{k'}^{(l)}\right) = 0 ,
\end{equation}
for all layers $l$, and for all $k$, $k'$ $\in [d_l]$
\end{lemma}
\begin{proof}
\begin{align}
    Cov\left(f_k^{(l)}, f_{k'}^{(l)}\right) &= Cov\left(\sum_{j=1}^{d_{l-1}}f_j^{(l-1)}W_{jk}^{(l)}, \sum_{j'=1}^{d_{l-1}} f_{j'}^{(l-1)}W_{j'k'}^{(l)}\right) \\
    &= \mathbb{E}\left[\sum_{j=1}^{d_{l-1}}f_j^{(l-1)}W^{(l)}_{jk} \sum_{j'=1}^{d_{l-1}} f_{j'}^{(l-1)}W^{(l)}_{j'k'} \right] \\
    &= \sum_{j=1}^{d_{l-1}}\sum_{j'=1}^{d_{l-1}} \mathbb{E}\left[f_j^{(l-1)}W_{jk}^{(l)} f_{j'}^{(l-1)}W_{j'k'}^{(l)} \right] \\
    &= \sum_{j=1}^{d_{l-1}}\sum_{j'=1}^{d_{l-1}} \mathbb{E}\left[f_j^{(l-1)}f_{j'}^{(l-1)} \right] \mathbb{E}\left[W^{(l)}_{jk}\right] \mathbb{E}\left[W^{(l)}_{j'k'} \right] \\
    &= 0 .
\end{align}
\end{proof}

However, they are \textbf{not} independent, but only conditionally independent given the previous later's units. As a remark, note that as $d_1 \rightarrow \infty$, the units $f_k^{(2)}$ approach a Gaussian distribution, for which uncorrelation implies independence.

\subsection{Main technical Lemma for induction}
Here we prove the main technical Lemma (Lemma \ref{lemma:integral_for_induction}) that allows us to perform the inductive step.
\begin{tcolorbox}[boxrule=0pt, sharp corners]
\begin{lemma*}
\label{lemma:proof_integral_for_induction}
Let $\bm{f}^{l}$ and $\bm{f}^{l-1}$ be a $d_l$-dimensional and a $d_{l-1}$-dimensional vectors, respectively. Let $\sigma_w^2 > 0$, $\tilde{\sigma}^2 > 0$ be two variance parameters, and $b_1, \dots b_{l-1} \in \mathbb{R}$. Then the following integral:
\begin{equation}
    I := \int_{\mathbb{R}^{d_{l-1}}} \frac{1}{( ||\bm{f}^{(l-1)}||^2)^{\frac{d_l}{2}}} \text{e}^{-\frac{||\bm{f}^{(l)}||^2}{2 \sigma_w^2||\bm{f}^{(l-1)}||^2}} G_{0,l-1}^{l-1,0}\left(\frac{|| \bm{f}^{(l-1)}||^{2}}{2^{l-1}\tilde{\sigma}^2 } \bigg\rvert b_1, \dots, b_{l-1}\right) d\bm{f}^{l-1} ,
\end{equation}
has solution:
\begin{align}
     I = C G_{0,l}^{l,0}\left(\frac{|| \bm{f}^{(l)}||^2}{2^{l}\sigma^2 } \bigg\rvert 0, \frac{1}{2}(d_{l-1} - d_l) + b_1 , \dots, \frac{1}{2}(d_{l-1} - d_l) + b_{l-1}\right) ,
\end{align}
where $\sigma^2 := \sigma_w^2 \tilde{\sigma}^2$, and $C:= \frac{1}{2}\tilde{C}2^{\frac{1}{2}(d_{l-1} - d_l)(l-1)}\tilde{\sigma}^{(d_{l-1} - d_l)}$, where $\tilde{C}$ is a constant that depends only on $d_{l-1}$.
\end{lemma*}
\end{tcolorbox}
\begin{proof}
The proof is based in two steps: in the first steps, we will write the integral in hyper-spherical coordinates. In the second step, we will apply a useful substitution and the properties of the Meijer-G function to solve the integral.
\paragraph{1. Hyper-spherical coordinates} Apply the following substitution:
\begin{align}
    & f_1^{(l-1)} = r\cos(\gamma_1) \\
    & f_2^{(l-1)} = r\sin(\gamma_1)\cos(\gamma_2) \\
    & \cdots \\
    & f_{d_{l-1}}^{(l-1)} = r \sin(\gamma_1)\cdots \sin(\gamma_{d_{l-1} - 1}) ,
\end{align}
where $r \in \mathbb{R}_{\geq 0}$ is the radius and $\gamma_1, \dots \gamma_{d_{l-1} -2} \in [0, \pi]$ and $\gamma_{d_{l-1} - 1} \in [0, 2\pi]$. 
The Jacobian is: 
\begin{scriptsize}
\begin{equation}
    J_n = \begin{pmatrix}
    \cos(\gamma_1) & -r\sin(\gamma_1) & 0 & 0 & \dots & 0 \\
    \sin(\gamma_1)\cos(\gamma_2) & r\cos(\gamma_1)\cos(\gamma_2) & - r\sin(\gamma_1)\sin(\gamma_2) & 0 & \dots & 0 \\
    \dots & & & & \dots & \dots \\
    r \sin(\gamma_1)\cdots \sin(\gamma_{d_{l-1} - 1}) & & & & & r \sin(\gamma_1)\cdots \cos(\gamma_{d_{l-1} - 1})
    \end{pmatrix}
\end{equation}
\end{scriptsize},
where it can be shown that its determinant is:
\begin{equation}
    |J_n| = r^{d_{l-1}-1}\sin^{d_{l-1} - 2}(\gamma_1) \sin^{d_{l-1}-3}(\gamma_2) \cdots \sin(\gamma_{d_{l-1} - 2}) .
\end{equation}
Therefore:
\begin{equation}
    \prod_{i=1}^{d_{l-1}}df_i^{(l-1)} = r^{d_{l-1}-1}\sin^{d_{l-1} - 2}(\gamma_1) \sin^{d_{l-1}-3}(\gamma_2) \cdots \sin(\gamma_{d_{l-1} - 2}) dr d\gamma_1 \cdots d\gamma_{d_{l-1} - 1} .
\end{equation}
By noting that the integral we are trying to solve depends only on $|| \bm{f}^{(l-1)} ||^2 = r^2$, we have that the density is, up to a normalization constant independent of $\bm{f}^{(l)}$:

\begin{align}
    I &= \tilde{C}\int r^{d_{l-1}-d_l-1} \text{e}^{-\frac{||\bm{f}^{(l)}||^2}{2 \sigma_w^2 r^2}} G_{0,l-1}^{l-1,0}\left(\frac{r^2}{2^{l-1}\tilde{\sigma}^{2} } \bigg\rvert b_1, \dots, b_{l-1}\right) dr \\
    &= \tilde{C}\int r^{d_{l-1}-d_l-1} G_{0,1}^{1,0}\left(\frac{|| \bm{f}^{(l)}||^2}{2\sigma_w^2 r^2} \bigg\rvert 0\right) G_{0,l-1}^{l-1,0}\left(\frac{r^2}{2^{l-1}\tilde{\sigma}^{2} } \bigg\rvert b_1, \dots, b_{l-1}\right) dr  ,
\end{align}
where we call $\tilde{C}$ the angular constant due to the integration of the angle-related terms (that do not depend on $r$, but only on $d_{l-1}$). We compute the angular constant in Lemma \ref{lemma:linear_net_distr_norm_const}.  In the last step we have applied the identity between the exponential function and the Meijer-G function as in Proposition \ref{prop:meijer_g_identities}.

\paragraph{2: Substitution and Meijer-G properties}
Defining $d = \frac{1}{2}(d_{l-1} - d_l)$, and applying the substitution $x = \frac{r^2}{2^{l-1}\tilde{\sigma}^{2}}$:
\begin{align}
     I &= \tilde{C}  \int (2^{\frac{l-1}{2}}\tilde{\sigma} x^{\frac{1}{2}})^{2d-1} G_{0,1}^{1,0}\left(\frac{|| \bm{f}^{(l)}||^2}{2^{l}\sigma_w^2 \tilde{\sigma}^{2} x} \bigg\rvert 0\right) G_{0,l-1}^{l-1,0}\left(x \bigg\rvert b_1, \dots, b_{l-1}\right) 2^{\frac{l-3}{2}}\tilde{\sigma} x^{-\frac{1}{2}}dx \\
    &= \frac{1}{2}\tilde{C}2^{d(l-1)}\tilde{\sigma}^{2d} \int  x^{d-1} G_{0,1}^{1,0}\left(\frac{|| \bm{f}^{(l)}||^2}{2^l\sigma_w^2\tilde{\sigma}^2 x} \bigg\rvert 0\right) G_{0,l-1}^{l-1,0}\left(x \bigg\rvert b_1,\dots, b_{l-1}\right)dx .
\end{align}

Defining $\sigma^2 = \sigma_w^2 \tilde{\sigma}^2$, $a^2:=\frac{|| \bm{f}^{(l)}||^2}{2^l\sigma^2 }$ and $C:= \frac{1}{2}\tilde{C}2^{d(l-1)}\tilde{\sigma}^{2d}$, and expanding the $G_{0,1}^{1,0}$ term according to the definition, we get:

\begin{align}
\label{eq:star_thm_linear_network}
    I &= C \int  x^{d-1} \frac{1}{2\pi i} \int \Gamma(s) \left(\frac{a^2}{x}\right)^{-s} G_{0,l-1}^{l-1,0}\left(x \bigg\rvert b_1,\dots, b_{l-1}\right)ds dx \\
    &= C \frac{1}{2\pi i} \int \Gamma(s) a^{-2s}  \int x^{s + d - 1} G_{0,l-1}^{l-1,0}\left(x \bigg\rvert b_1,\dots, b_{l-1}\right)dx ds  .
\end{align}
where we can change the order of integration due to the fact that the integrand is positive in the integration region (Tonelli's theorem).
Now by using Proposition \ref{prop:mellin_transf_g}, the inner integral has the following solution:
\begin{align}
    \int x^{s + d - 1} G_{0,l-1}^{l-1,0}\left(x \bigg\rvert b_1, \dots b_{l-1}\right)dx &= \prod_{i=1}^{l-1} \Gamma(d + b_i +s) \\
    &= \prod_{i=1}^{l-1} \Gamma\left(\frac{1}{2}(d_{l-1} - d_l) + b_i +s\right) .
\end{align}

Therefore we can conclude that:
\begin{align}
    I &= C \frac{1}{2\pi i} \int \Gamma(s) \prod_{i=1}^{l-1} \Gamma( \frac{1}{2}(d_{l-1} - d_l) + b_i +s) a^{-2s} ds \\
    &= C G_{0,l}^{l,0}\left(a^2 \bigg\rvert 0, \frac{1}{2}(d_{l-1} - d_l) + b_1 , \dots, \frac{1}{2}(d_{l-1} - d_l) + b_{l-1} \right) \\
    &= C G_{0,l}^{l,0}\left(\frac{|| \bm{f}^{(l)}||^2}{2^{l}\sigma^2 } \bigg\rvert0, \frac{1}{2}(d_{l-1} - d_l) + b_1 , \dots, \frac{1}{2}(d_{l-1} - d_l) + b_{l-1} \right) ,
\end{align}
where we have simply applied the definition of the Meijer-G function.

\end{proof}

\subsection{Probability Density Function for linear networks}
\label{sec:proof_thm_linear_net_any_width}
We proof the result on the probability density function for a linear network in the following.
\begin{tcolorbox}[boxrule=0pt, sharp corners]
\begin{theorem*}
Suppose $l \geq 1$, and the input has dimension $d_0$. Then, the joint marginal density of the random vector $\bm{f}^{(l)}$ (i.e.\ the density of the $l$-th layer pre-activations) is proportional to:
\begin{equation}
    p(\bm{f}^{(l)}) \propto G^{l, 0}_{0,l}\left(\frac{||\bm{f}^{(l)}||^2}{2^l  \sigma^2} \bigg \rvert  0, \frac{1}{2}\left(d_1 - d_l\right), \dots, \frac{1}{2}\left( d_{l-1} - d_l\right) \right) ,
\end{equation}
where $\sigma^2 = \prod_{i=1}^l \sigma_{i}^2$ .
\end{theorem*}
\end{tcolorbox}
\begin{proof}
We proof by induction. For the base case, consider $l=1$. We have shown that 
\begin{equation}
    \bm{f}^{(1)} \sim \mathcal{N}(0, \sigma_1^2I) .
\end{equation}
Therefore we can re-write its density as:
\begin{align}
    p(f^{(1)}, \dots f^{(d)}) &= \frac{1}{\left(2\pi \sigma_1^2\right)^{\frac{d_1}{2}}} \exp\left(-\frac{||\bm{f}^{(1)}||^2}{2 \sigma_1^2} \right) \\
    &= \frac{1}{\left(2\pi \sigma_1^2\right)^{\frac{d_1}{2}}} G^{1,0}_{0,1}\left(\frac{||\bm{f}^{(1)}||^2}{2 \sigma_1^2} \bigg \rvert 0 \right) ,
\end{align} 
where we have used the identity between the exponential function and the Meijer-G function (Proposition \ref{prop:meijer_g_identities}).

Now let $\tilde{\sigma}^2 = \prod_{i=1}^{l-1}\sigma_i^2$. Assume that 
\begin{equation}
    p(f^{(l)}_1, \dots, f^{(l-1)}_{d_{l-1}}) \propto G^{l-1, 0}_{0,l-1}\left(\frac{||\bm{f}^{(l-1)}||^2}{2^{l-1}  \tilde{\sigma}^2} \bigg \rvert  0, \frac{1}{2}\left(d_1 - d_{l-1}\right), \dots, \frac{1}{2}\left( d_{l-2} - d_{l-1}\right) \right) .
\end{equation} 

Now we can use the fact that the the units of the $l$-th layer are conditionally independent given the previous' layer units. Furthermore the conditional distribution is Gaussian due to the fact that the weights are i.i.d Gaussian. Therefore we can write:
\begin{align}
    p(f_1^{(l)}, \dots f_{d_l}^{(l)}) &= \int_{\mathbb{R}^{d_{l-1}}} p(f_1^{(l)}, \dots f_{d_l}^{(l)} | f_1^{(l-1)}, \dots f_{d_{1}}^{(l-1)} ) p(f_1^{(l-1)}, \dots f_{d_1}^{(l-1)}) d\bm{f}^{(l-1)}\\
    &\propto \int_{\mathbb{R}^{d_{l-1}}} \frac{1}{(2\pi ||\bm{f}^{(l-1)}||^2)^{\frac{d_l}{2}}} \text{e}^{-\frac{||\bm{f}^{(l)}||^2}{2 \sigma_l^2||\bm{f}^{(l-1)}||^2}} \\ 
    & G_{0,l-1}^{l-1,0}\left(\frac{|| \bm{f}^{(l-1)}||^{2}}{2^{l-1}\tilde{\sigma}^2 } \bigg\rvert  0, \frac{1}{2}\left(d_1 - d_{l-1}\right), \dots, \frac{1}{2}\left( d_{l-2} - d_{l-1}\right)\right) d\bm{f}^{(l-1)} .
\end{align}
In the first step we have marginalized out the units of the $l-1$ layer, and applied the product rule of probabilities.  In the second step we have applied the induction hypothesis. 

The integral is in the form of Lemma \ref{lemma:integral_for_induction}. For the coefficients of the Meijer-G function $b_2 = \frac{1}{2}\left(d_1 - d_{l-1}\right)$, $\dots$, $b_{l-1} = \frac{1}{2}\left(d_{l-2} - d_{l-1}\right)$, note that:
\begin{equation}
    \frac{1}{2}\left(d_i - d_{l-1}\right) + \frac{1}{2}\left(d_{l-1} - d_l \right) = \frac{1}{2}\left(d_i - d_l \right)
\end{equation}
holds for all $i \in [d_{l-2}]$ and clearly $b_1 + \frac{1}{2}\left(d_{l-1} - d_l \right) = \frac{1}{2}\left(d_{l-1} - d_l \right)$ as $b_1 = 0$ in our case.
Therefore by Lemma \ref{lemma:integral_for_induction} we can conclude that:
\begin{align}
    p(f_1^{(l)}, \dots f_{d_l}^{(l)}) &\propto G_{0,l}^{l,0}\left(\frac{|| \bm{f}^{(l)}||^2}{2^{l}\sigma^{2} } \bigg\rvert0, \frac{1}{2}(d_1 - d_{l}) \dots, \frac{1}{2}(d_{l-1} - d_{l}) \right) .
\end{align}
\end{proof}
\subsection{CDF of prior predictive}
We also derive the CDF of the linear network in the following theorem and proceed to prove it.
\begin{tcolorbox}[boxrule=0pt, sharp corners]
\begin{theorem}[CDF of prior predictive]
    Let $f^{l}$ be the output of a of a linear network of $l$ layers. We assume the final layer is one dimensional. Then the the cdf is $F_l(t) := 1 - P(f^{l} > t)$, $t > 0$. We have that
    \begin{equation}
         P(f^{l} > t) = \frac{t}{2C}G^{l+1,0}_{1,l+1}\left(\omega t^2 \bigg \rvert \begin{smallmatrix}\frac{1}{2} & & & &\\ -\frac{1}{2} & 0 & b_1 & \dots & b_{l-1}  \end{smallmatrix} \right)  ,
    \end{equation}
    where $b_i = \frac{1}{2}(d_i - 1)$, $i \in [l-1]$, $C$ is the normalization constant and $\omega = \frac{1}{2^l  \sigma^{2}}$.
\end{theorem}
\end{tcolorbox}
\begin{proof}
    Let $X = f^{l}$.
    \begin{align}
        P(X > t) &= \frac{1}{C}\int_t^{\infty} G^{l,0}_{0,l}\left(\omega x^2 \bigg \rvert 0,  b_1, \dots, b_{l-1} \right)dx \\
        &= \frac{t}{2C} \int_{1}^{\infty} y^{-\frac{1}{2}} G^{l,0}_{0,l}\left(\omega t^2 y \bigg \rvert 0,  b_1, \dots, b_{l-1} \right)dx \\
        &= \frac{t}{2C}G^{l+1,0}_{1,l+1}\left(\omega t^2 \bigg \rvert \begin{smallmatrix}\frac{1}{2} & & & &\\ -\frac{1}{2} & 0 & b_1 & \dots & b_{l-1}  \end{smallmatrix} \right),
    \end{align}
    where in the first step we have used the result of \ref{thm:distr_linear_network_any_width}, in the second step we have applied the substitution $y = \frac{x^2}{t^2}$, and in the last step we have used Equation \ref{eq:meijer_g_property_for_cdf} with $\rho=\frac{1}{2}$, $\sigma=1$ and $\alpha=\omega t^2$
\end{proof}

\subsection{Resulting Moments for Linear Networks}
 
\label{sec:proof_moments_linear}

Define $\omega = \frac{1}{2^l  \sigma^{2}}$. Denote by $\tilde{p}$ the unnormalized measure and
define the random variable
\begin{equation}
    Z = ||\bm{f}^{(l)}||_2^2 .
\end{equation}

We are interested in the k-th moment of $Z$. Using spherical coordinates and the properties of the Meijer-G function in a similar way as the proofs above, we get:

\begin{align}
        \mathbb{E}\left[Z^k\right] &= \frac{1}{C}\int_{\mathbb{R}^{m}}||\bm{z}||_2^{2k}G_{0,l}^{l,0} | \left(\omega ||\bm{z}||_2^{2}\bigg\rvert  0, \frac{1}{2}\left(d_1 - d_l\right), \dots, \frac{1}{2}\left( d_{l-1} - d_l \right)\right)d\bm{z}\\
        &= \frac{\tilde{C}_l}{C} \int_{0}^{\infty}r^{2k + d_l - 1}G_{0,l}^{l,0}\left(\omega r^{2}\bigg\rvert  0, \frac{1}{2}\left(d_1 - d_l\right), \dots, \frac{1}{2}\left( d_{l-1} - d_l \right) \right) dr \\
        &= \frac{\tilde{C}_l}{2C}\int_{0}^{\infty}x^{k + \frac{d_l}{2} - 1}G_{0,l}^{l,0}\left(\omega x \bigg\rvert  0, \frac{1}{2}\left(d_1 - d_l\right), \dots, \frac{1}{2}\left( d_{l-1} - d_l \right)\right)dx \\
        &= \frac{\tilde{C}_l}{2C} \omega^{-k -\frac{d_l}{2}}\prod_{i=1}^{l}\Gamma\left(\frac{d_i}{2}+k\right) \\
        &= \omega^{-k} \frac{\prod_{i=1}^{l}\Gamma\left(\frac{d_i}{2}+k\right)}{\prod_{i=1}^{l}\Gamma\left(\frac{d_i}{2}\right)} \\
        &= \left(2^{l}\sigma^{2}\right)^k \prod_{i=1}^{l} \frac{\Gamma\left(\frac{d_i}{2}+k\right)}{\Gamma\left(\frac{d_i}{2}\right)} .
\end{align}
Note that it can be equivalently written as:
\begin{align}
     \mathbb{E}\left[Z^k\right] &= (2^l \sigma^2)^{k-1} \prod_{i=1}^{l}\left(\frac{ \Gamma\left(\frac{d_i}{2}+k - 1\right)}{\Gamma\left(\frac{d_i}{2}\right)}\right) (2^l \sigma^2) \prod_{i=1}^l \left(\frac{d_i}{2} + k - 1\right) \\
     &= \mathbb{E}\left[Z^{k-1}\right] \sigma^2 \prod_{i=1}^l \left(d_i + 2(k - 1)\right) ,
\end{align} 
so the kurtosis is:
\begin{align}
    \kappa = \frac{\prod_{i=1}^l \left(d_i + 2(2 - 1)\right)}{\prod_{i=1}^l d_i} = \prod_{i=1}^l \frac{d_i + 2}{d_i}
\end{align}
If $d_1 = \dots = d_{l-1} = m$, and $d_l = 1$.
For instance the variance ($k=1$) is\footnote{By symmetry, all the odd moments are zero}:
\begin{equation}
\label{eq:variance_linear_net}
    \left(2^{l}\sigma^{2}\right) \left(\frac{\Gamma(\frac{m}{2} + 1)}{\Gamma(\frac{m}{2})} \right)^{l-1} \frac{\Gamma(1 + \frac{1}{2})}{\Gamma(\frac{1}{2})} = (2^l  \sigma^{2}) \frac{m^{l-1}}{2^{l-1}} \frac{1}{2} = \sigma^{2}m^{l-1}.
\end{equation}
\subsection{Infinite width and depth limit}
We also present the infinite-width and infinite-depth result for the linear case. Due to the linear nature, the proof simplifies significantly compared to the ReLU case. 

\begin{tcolorbox}[boxrule=0pt, sharp corners]
\begin{lemma}
\label{moments_linear}
Consider the distribution of the output $p(f^{(L)})$, as defined in Thm.~\ref{thm:distr_linear_network_any_width}. Denote  $X \sim \mathcal{N}(0,1)$, $Y \sim \mathcal{LN}(-\frac{\gamma}{2}, \frac{\gamma}{2})$ and the normal log-normal mixture $Z=XY$. \\[2mm]For fixed depth $L \in \mathbb{N}$, under NTK parametrization , it holds that 
\begin{equation}
    p^m(f^{(L)}) \xrightarrow[]{d} X \hspace{4mm} \text{for } m \xrightarrow[]{} \infty
\end{equation}

In contrast, for growing depth $L=\gamma m$, we have the following convergence of the moments
\begin{equation}
    \mathbb{E}\left[\left(f^{(L)}\right)^{2k}\right] \xrightarrow[]{m \xrightarrow[]{}\infty} \mathbb{E}[Z^{2k}] = e^{\gamma k(k-1)}(2k-1)!!
\end{equation}

where $(2k-1)!! = (2k-1) \dots 3 \cdot 1$ denotes the double factorial.
\end{lemma}
\end{tcolorbox}
\begin{proof}
Recall that the moments of $||\bm{f}^{(l)}{||}_2$ are given by
\begin{equation}
    \mathbb{E}\left[||\bm{f}^{(l)}{||}_2^{2k}\right] = \left(2^{l}\sigma^{2}\right)^k \prod_{i=1}^{l} \frac{\Gamma\left(\frac{d_i}{2}+k\right)}{\Gamma\left(\frac{d_i}{2}\right)} ,
\end{equation}

where $\sigma^2 = \prod_{i=1}^{l}\sigma_i^2$. 
Assuming $d_1 = \dots d_{l-1}=m$ and $d_l=1$ and the NTK parametrization $\sigma_1^2 =1$ and $ \sigma_2^2=\dots= \sigma_{l}^2 = \frac{1}{m}$ simplifies this to
\begin{align}
        \mathbb{E}\left[\left({f}^{(l)}\right)^{2k}\right] &= \left(\frac{2^{l}}{m^{l-1}}\right)^k \frac{\Gamma(\frac{m}{2} + k)^{l-1}}{\Gamma(\frac{m}{2})^{l-1}}\frac{\Gamma(\frac{1}{2} + k)}{\Gamma(\frac{1}{2})} \\&= \left(\frac{2^{l}}{m^{l-1}}\right)^k \frac{\Gamma(\frac{m}{2} + k)^{l-1}}{\Gamma(\frac{m}{2})^{l-1}}2^{-k}(2k-1)!! \\
        &=  \left(\frac{2^{k}}{m^{k}}\frac{\Gamma(\frac{m}{2} + k)}{\Gamma(\frac{m}{2})}\right)^{l-1}(2k-1)!!\\
        &=  \left(\frac{2^{k}}{m^{k}}\left(\frac{m}{2} + k-1\right) \dots \left(\frac{m}{2} +1\right) \frac{m}{2}\right)^{l-1}(2k-1)!! .
\end{align}
Define the $k$-th order polynomial $p(m) = \left(\frac{m}{2} + k-1\right) \dots \left(\frac{m}{2} +1\right) \frac{m}{2}$. Denote its coefficients by $\alpha_i$ for $i=1, \dots, k$. We know that $\alpha_k = \frac{1}{2^{k}}$ and from Lemma~\ref{coef_poly} that
\begin{equation}
    \alpha_{k-1} = \frac{1}{2^{k-1}}\sum_{i=1}^{k}(k-i) = \frac{1}{2^{k-1}}\left(k^2 - \frac{k(k+1)}{2}\right) =  \frac{k^2-k}{2^k} .
\end{equation}

Assuming constant depth, performing the division by $m^{k}$ thus leads to
\begin{equation}
    \begin{split}
        & \left(2^k\left(\alpha_k + \alpha_{k-1}\frac{1}{m} + \dots + \mathcal{O}\left(\frac{1}{m^2}\right)\right) \right)^{l-1}(2k-1)!! \\
        &= \left(1 + \frac{k(k-1)}{m} + \mathcal{O}\left(\frac{1}{m^2}\right)\right)^{l-1}(2k-1)!!\\
        &= \left(1 + \frac{(l-2)\left((k-1)k\right)}{m} + \mathcal{O}\left(\frac{1}{m^2}\right)\right)(2k-1)!! .
    \end{split}
\end{equation}
Now we can easily see that 
\begin{equation}
    \mathbb{E}\left[\left({f}^{(l)}\right)^{2k}\right] \xrightarrow[]{m \xrightarrow[]{}\infty}(2k-1)!! .
\end{equation}
Recall that for $X \sim \mathcal{N}(0,1)$ we have the same moments $\forall k \in \mathbb{N}$: $\mathbb{E}\left[X^{2k}\right] = (2k-1)!!$, whereas the odd moments vanish for both distributions due to symmetry. The convergence of the moments, due to \citet{Bill86} and the identifiability of the Gaussian distribution implies convergence in distribution. \\[3mm]
On the other hand, if we assume that depth grows proportional to width, i.e.\ $l-1 = \gamma m$ for $\gamma >0$, we arrive at a different limit given by
\begin{equation}
    \begin{split}
        \left(1 + \frac{k(k-1)}{m} + \mathcal{O}\left(\frac{1}{m^2}\right)\right)^{\gamma m}(2k-1)!! \xrightarrow[]{m \xrightarrow[]{}\infty} e^{\gamma k(k-1)}(2k-1)!! .
    \end{split}
\end{equation}
Consider the random variable $Z = X Y$ where $X \sim \mathcal{N}(0,1)$ and $Y \sim \mathcal{LN}(s, t^2)$ are two independent variables. For $k \in \mathbb{N}$, we can compute the moments as
\begin{equation}
    \mathbb{E}\left[Z^{n}\right]=\mathbb{E}\left[X^{n}Y^{n}\right] = \mathbb{E}\left[X^{n}\right]\mathbb{E}\left[Y^{n}\right] = \begin{cases}0  \hspace{33mm} $n$ \text{ odd}\\
    (2k-1)!!e^{2ks + 2k^2t^2} \hspace{5mm} $n = 2k$\end{cases}
\end{equation}
Choosing $s = -\frac{\gamma}{2}$ and $t^2 = \frac{\gamma}{2}$ hence recovers the moments exactly.
\end{proof}

\subsection{Normalization Constant and Angular Constant}
\label{sec:normalization_constant}

We complete the picture by calculating the normalization constant of the resulting distribution.
\begin{tcolorbox}[boxrule=0pt, sharp corners]
\begin{lemma}[normalization constant]
\label{lemma:linear_net_distr_norm_const}
Under the conditions of Theorem \ref{thm:distr_linear_network_any_width}, the normalization constant $C$ for the density of the l-th layer can be computed as:
\begin{equation}
    C = \frac{1}{2} \tilde{C}_l\left(\frac{1}{2^{l}\sigma^{2}}\right)^{-\frac{d_l}{2}}\prod_{i=1}^l\Gamma\left(\frac{d_i}{2}\right) ,
\end{equation}
or, expanding $\tilde{C}_l$ according to Lemma \ref{LEMMA:ANGLULAR_CONSANT_LINEAR_NET}:
\begin{equation}
    \frac{\pi^{\frac{d_l}{2}}}{\Gamma\left(\frac{d_l}{2} \right)} \left(\frac{1}{2^{l}\sigma^{2}}\right)^{-\frac{d_l}{2}}\prod_{i=1}^l\Gamma\left(\frac{d_i}{2}\right) .
\end{equation}
\end{lemma}
\end{tcolorbox}
\begin{proof}[proof of lemma \ref{lemma:linear_net_distr_norm_const}]
The normalization constant has the following form:
\begin{align}
    C=\int_{\mathbb{R}^{d_l}}\tilde{p}({\bm{f}^{(l)}})d\bm{f}^{(l)} &= \int_{\mathbb{R}^{d_l}}G_{0,l}^{l,0}\left(\frac{|| \bm{f}^{(l)}||^2}{2^{l}\sigma^{2} } \bigg\rvert0, \frac{1}{2}(d_1 - d_{l}) \dots, \frac{1}{2}(d_{l-1} - d_{l}) \right)d\bm{f}^{(l)} \\
    &= \tilde{C}_l \int_{0}^{\infty} r^{d_l-1}G_{0,l}^{l,0}\left(\frac{r^2}{2^{l}\sigma^{2} } \bigg\rvert0, \frac{1}{2}(d_1 - d_{l}) \dots, \frac{1}{2}(d_{l-1} - d_{l}) \right)dr \\
    &= \frac{1}{2}\tilde{C}_l\int_{0}^{\infty}x^{\frac{d_l}{2}-1}G_{0,l}^{l,0}\left(\frac{x}{2^{l}\sigma^{2} } \bigg\rvert0, \frac{1}{2}(d_1 - d_{l}) \dots, \frac{1}{2}(d_{l-1} - d_{l}) \right)dx\\
    &=\frac{1}{2} \tilde{C}_l\left(\frac{1}{2^{l}\sigma^{2}}\right)^{-\frac{d_l}{2}}\prod_{i=1}^l\Gamma\left(\frac{d_i}{2}\right) ,
\end{align}

where we used spherical coordinates and the substitution $x=r^2$. We denote the angular constant by $\tilde{C}_l$, and according to Lemma \ref{LEMMA:ANGLULAR_CONSANT_LINEAR_NET} has solution: $\tilde{C}_l = \frac{2\pi^\frac{d_l}{2}}{\Gamma(\frac{d_l}{2})}$.

\end{proof}
\subsection{Angular constant}
\begin{tcolorbox}[boxrule=0pt, sharp corners]
\begin{lemma}
\label{LEMMA:ANGLULAR_CONSANT_LINEAR_NET}
The angular constant:
\begin{equation}
    \tilde{C_l} = \int_{0}^{2\pi} d\gamma_{d_l-1} \int_{0}^\pi \sin^{d_l - 2}(\gamma_{1}) d\gamma_1 \int_{0}^\pi \sin^{d_l - 3}(\gamma_2) d\gamma_2 \dots \int_{0}^\pi \sin(\gamma_{d_l - 2}) d\gamma_{d_l - 2} 
\end{equation}
has solution:
\begin{equation}
    \tilde{C_l} = \frac{2 \pi^{\frac{d_l}{2}}}{\Gamma\left(\frac{d_l}{2} \right)} .
\end{equation}
\end{lemma}
\end{tcolorbox}
\begin{proof}
The angular constant $\tilde{C}_l$ can be calculated as follows (for $d_l \geq 2$):
\begin{align*}
    \tilde{C_l} &= \int_{0}^{2\pi} d\gamma_{d_l-1} \int_{0}^\pi \sin^{d_l - 2}(\gamma_{1}) d\gamma_1 \int_{0}^\pi \sin^{d_l - 3}(\gamma_2) d\gamma_2 \dots \int_{0}^\pi \sin(\gamma_{d_l - 2}) d\gamma_{d_l - 2} \\
    &= 2\pi \prod_{k=1}^{d_l - 2} \int_{0}^\pi \sin^{d_l - k - 1}(\gamma_{k}) d\gamma_{k} \\
    &= 2\pi \prod_{k=1}^{d_l - 2} \frac{\Gamma(d_l - k -1)}{2^{d_l - k -1} \Gamma\left(\frac{d_l -k - 1}{2}\right)\Gamma\left(\frac{d_l -k - 1}{2} + 1\right)}(2\pi) \\
    &= \frac{(2\pi)^{d_l - 1}}{2^{\frac{1}{2}(d_l - 2)(d_l -1)}} \prod_{k=1}^{d_l - 2} \frac{\Gamma(d_l - k -1)}{ \Gamma\left(\frac{d_l -k - 1}{2}\right)\Gamma\left(\frac{d_l -k - 1}{2} + 1\right)} ,
\end{align*}

where we have used Lemma \ref{LEMMA:INT_POWERS_SIN} to compute the integrals. If $d_l = 1$, then there is no need to write the integral in spherical coordinates and we can simply set $\tilde{C} = 1$.

Now we can apply the Legendre duplication formula:
\begin{equation}
    \Gamma(z)\Gamma(z + \frac{1}{2}) = 2^{1 - 2z} \sqrt{\pi} \Gamma(2z)
\end{equation}
to the numerator term $2z = d_l - k - 1$ and get:
\begin{equation}
    \tilde{C_l} = \frac{(2\pi)^{d_l - 1}}{\pi^{\frac{d_l - 2}{2}}2^{\frac{1}{2}(d_l - 2)(d_l -1)}} \prod_{k=1}^{d_l - 2} \frac{\cancel{\Gamma\left(\frac{d_l -k - 1}{2}\right)} \Gamma(\frac{d_l - k -1}{2} + \frac{1}{2})2^{d_l - k - 2}}{ \cancel{\Gamma\left(\frac{d_l -k - 1}{2}\right)}\Gamma\left(\frac{d_l -k - 1}{2} + 1\right)} .
\end{equation}
Note that that the product:
\begin{equation}
    \prod_{k=1}^{d_l - 2} \frac{ \Gamma(\frac{d_l - k -1}{2} + \frac{1}{2})}{ \Gamma\left(\frac{d_l -k - 1}{2} + 1\right)} = \frac{\Gamma(\frac{d_l - 1}{2})}{\Gamma(\frac{d_l}{2})}\cdot \frac{\Gamma(\frac{d_l - 2}{2})}{\Gamma(\frac{d_l - 1}{2})} \cdots \frac{1}{\Gamma(\frac{1}{2})} = \frac{1}{\Gamma(\frac{d_l}{2})}
\end{equation}
Finally, we have the product
\begin{equation}
    \prod_{k=1}^{d_l - 2} 2^{d_l - k - 2} = 2^{d_l(d_l-2) - \frac{(d_l - 2)(d_l - 1)}{2} - 2(d_l - 2)} ,
\end{equation}
from which we can conclude, after some elementary algebraic manipulations:
\begin{equation}
    \tilde{C}_l = \frac{2 \pi^{\frac{d_l}{2}}}{\Gamma\left(\frac{d_l}{2} \right)} .
\end{equation}
\end{proof}
Finally we prove the technical Lemma that we used in the previous proof.
\begin{tcolorbox}[boxrule=0pt, sharp corners]
\begin{lemma}
\label{LEMMA:INT_POWERS_SIN}
\begin{equation}
    \int_{0}^\pi \sin^{k}(x) dx = \frac{\Gamma(k)}{2^k \Gamma\left(\frac{k}{2}\right)\Gamma\left(\frac{k}{2} + 1\right)}(2\pi) .
\end{equation}
\end{lemma}
\end{tcolorbox}
\begin{proof}
By integrating by parts, and using some algebraic manipulation, it is easy to see that:
\begin{equation}
    \int \sin^{k}(x) dx = -\frac{1}{k} \sin^{k-1}x \cos x + \frac{k-1}{k}\int \sin^{k-2} x dx .
\end{equation}
Evaluating the integral between $0$ and $\pi$, we get:
\begin{equation}
    \int_{0}^\pi \sin^{k}(x) dx = \frac{k-1}{k}\int_{0}^\pi \sin^{k-2} x dx .
\end{equation}
By unrolling the recursion:
\begin{equation}
    \int_{0}^\pi \sin^{k}(x) dx = \frac{(k-1)(k-3)\cdots}{k(k-2)\cdots} \begin{cases}
    \int_{0}^\pi dx = \pi & \text{if $k$ is even} \\
    \int_{0}^\pi \sin(x)dx = 2 & \text{if $k$ is odd}
    \end{cases} .
\end{equation}
The following expression includes both the even and the odd case:
\begin{equation}
        \int_{0}^\pi \sin^{k}(x) dx = \frac{\Gamma(k)}{2^k \Gamma\left(\frac{k}{2}\right)\Gamma\left(\frac{k}{2} + 1\right)}(2\pi)
\end{equation}
In fact, if $k$ is even, then:
\begin{align*}
     \frac{\Gamma(k)}{2^k \Gamma\left(\frac{k}{2}\right)\Gamma\left(\frac{k}{2} + 1\right)}(2\pi) &=  \frac{(k-1)(k-2)\cdots}{2^k \frac{k}{2}(\frac{k}{2} - 1)^2 (\frac{k}{2} -2 )^2\cdots} (2\pi)\\ 
     &= \frac{(k-1)(k-3)\cdots}{k(k-2)\cdots} \pi
\end{align*}
If $k$ is odd, we use the identity:
\begin{equation}
    \Gamma\left(\frac{k}{2}\right) = \frac{(k-2)!! \sqrt{\pi} }{2^{\frac{k-1}{2}}} ,
\end{equation}
where $k!!$ is the double factorial.
Following a very similar procedure, we get the desired result.
\end{proof}

\paragraph{Important Remark:} In the ReLU case, we will see that that the integral is from $0$ to $\frac{\pi}{2}$. In that case, we get:
\begin{equation}
    \int_{0}^{\frac{\pi}{2}} \sin^{k}(x) dx = \frac{\Gamma(k)}{2^k \Gamma\left(\frac{k}{2}\right)\Gamma\left(\frac{k}{2} + 1\right)}\pi .
\end{equation}
Therefore the angular constant is:

\begin{align}
    \tilde{C_l} &= \int_{0}^{\frac{\pi}{2}} d\gamma_{d_l-1} \int_{0}^{\frac{\pi}{2}} \sin^{d_l - 2}(\gamma_{1}) d\gamma_1 \int_{0}^{\frac{\pi}{2}} \sin^{d_l - 3}(\gamma_2) d\gamma_2 \dots \int_{0}^{\frac{\pi}{2}} \sin(\gamma_{d_l - 2}) d\gamma_{d_l - 2}  \\
    &= \frac{\pi}{2} \prod_{k=1}^{d_l - 2} \int_{0}^{\frac{\pi}{2}} \sin^{d_l - k - 1}(\gamma_{k}) d\gamma_{k} \\
    &= \frac{\pi^{d_l -1}}{2} \prod_{k=1}^{d_l - 2} \frac{1}{2^{d_l - k -1}}  \frac{\Gamma(d_l - k -1)}{ \Gamma\left(\frac{d_l -k - 1}{2}\right)\Gamma\left(\frac{d_l -k - 1}{2} + 1\right)} \\
    &= \frac{\pi^{d_l -1}}{2 \pi^{\frac{d_l - 2}{2}}} \prod_{k=1}^{d_l - 2} \frac{\Gamma(\frac{d_l - k -1}{2} + \frac{1}{2})2^{- 1}}{ \Gamma\left(\frac{d_l -k - 1}{2} + 1\right)} \\
    &= \frac{\pi^{d_l -1}}{2 \pi^{\frac{d_l - 2}{2}}} \frac{2^{-(d_l - 2)}}{\Gamma(\frac{d_l}{2})} \\
    &= \frac{ \pi^{\frac{d_l}{2}}}{2^{d_l - 1}\Gamma(\frac{d_l}{2})} .
\end{align}

\subsection{Kurtosis of Linear Networks}

Using the closed form expressions from Section \ref{sec:proof_moments_linear}, we can describe the kurtosis of the output $f^{(l)}$ as
\begin{equation}
\label{eq:kurtosis_linear_net}
    \kappa_{\text{lin}} := \frac{3  \sigma_w^{2l}(m+2)^{l-1}}{ \sigma_w^{2l} m^{l-1}} = 3\left(\frac{m+2}{m}\right)^{l-1} .
\end{equation}
In particular, the distribution is always more heavy-tailed than a Gaussian, for which $\kappa = 3$ (i.e.\ the distribution is leptokurtic). The second obvious conclusion is that depth increases the heavy-tailedness exponentially, which is in-line with the theoretical results of \citep{vladimirova2019understanding}. On the contrary, the width has the effect of "normalizing" the distribution, in particular in the limit of large width we have that:
\begin{equation}
    \lim_{m \rightarrow \infty} \kappa = 3 ,
\end{equation}
which is the kurtosis of the Gaussian distribution, as anticipated from Lemma~\ref{moments_linear}. 
\section{Proofs for ReLU networks and Derivation of their Moments}

Now, we extend these results to ReLU networks.
We need the following additional notation: we call $\delta(x - x_0)$ the Dirac delta function centered at $x_0$, and $\mathbbm{1}_A := \begin{cases} 1 & x \in A \\ 0 & \text{else} \end{cases}$ the indicator function. Also, we indicate with $\mathcal{S}$ the set of indices $1, \dots, d$ that index $d$ random variables, and with $\Omega$ its power set, i.e.\, the set of all possible subsets of $\mathcal{S}$. Note that $|\Omega| = 2^{d}$.
We will use the following lemma, which explains what happens to a joint density when the marginals are transformed by the ReLU function. 

\label{relu_proofs}
\subsection{Effect of ReLU activation function on the joint density}
\label{sec:proof_lemma_relu_transform_density}
\begin{tcolorbox}[boxrule=0pt, sharp corners]
\begin{lemma}
\label{lemma:relu_transform_density}
Let $p(f_1, \dots f_d)$ be the joint density of the $d$ random variables $f_1, \dots f_d$.\footnote{We use the same notation for the random variable and the corresponding dummy variable in the density function.} Assume $p(f_1, \dots f_d)$ is symmetric around zero. When we apply the transformation $g_i = ReLU(f_i)$ to each $i \in [d]$, then the joint density of the transformed variables has the following form:
\begin{equation}
    p_{\text{ReLU}}(g_1, \dots g_d) = \sum_{A \in \Omega} \frac{1}{2^{|A|}} p(\bm{g}_{\mathcal{S} \setminus A}) \prod_{j \in \mathcal{S} \setminus A} \mathbbm{1}_{g_j > 0}\prod_{i \in A} \delta(g_i) ,
\end{equation}
where $p(\bm{g}_{\mathcal{S} \setminus A})$ is the marginal density of the random variables whose indexes are in $\mathcal{S} \setminus A$. 
\end{lemma}
\end{tcolorbox}
\begin{proof}
Again we use conditional independence: the activations $g_1, \dots, g_d$ are independent given the pre-activations $f_1, \dots, f_d$. So we can write:
\begin{align}
    p_{\text{ReLU}}(g_1, \dots, g_d) &= \int_{\mathbb{R}^d} p(g_1, \dots, g_d | f_1, \dots, f_d) p(f_1, \dots, f_d) d\bm{f} \\ 
    &=  \int_{\mathbb{R}^d} \prod_{i=1}^d p(g_i | f_i) p(f_1, \dots, f_d) d\bm{f} .
\end{align}
Now, $p(g_i | f_i) = \begin{cases} \delta(g_i) &  f_i < 0 \\ \delta(g_i - f_i) & f_i \geq 0  \end{cases} = \delta(g_i) \mathbbm{1}_{f_i < 0} + \delta(g_i-f_i)\mathbbm{1}_{f_i \geq 0}$. So we can write:

\begin{align}
    p_{\text{ReLU}}(g_1, \dots, g_d) &= \int_{\mathbb{R}^d} \prod_{i=1}^d \left( \delta(g_i) \mathbbm{1}_{f_i < 0} + \delta(g_i-f_i)\mathbbm{1}_{f_i \geq 0} \right) p(f_1, \dots, f_d) d\bm{f} \\
    &= \int_{\mathbb{R}^d} \prod_{i=1}^d \left( \delta(g_i) \mathbbm{1}_{f_i < 0} + \delta(g_i-f_i)\mathbbm{1}_{f_i \geq 0} \right) p(f_1, \dots, f_d) d\bm{f} \\
    &= \int_{\mathbb{R}^d} \sum_{A \in \Omega} \left( \prod_{i \in A} \delta(g_i)\mathbbm{1}_{f_i < 0} \prod_{j \in \mathcal{S} \setminus A} \delta(g_j - f_j)\mathbbm{1}_{f_j \geq 0} \right) p(f_{1}, \dots, f_{d}) d\bm{f} \\
    &= \sum_{A \in \Omega} \int_{\mathbb{R}^{|A|}_{< 0}} \int_{\mathbb{R}^{d - |A|}} \prod_{i \in A} \delta(g_i) \prod_{j \in \mathcal{S} \setminus A} \delta(g_j - f_j)\mathbbm{1}_{f_j \geq 0} p(f_{1}, \dots, f_{d}) d\bm{f} \\
    &= \sum_{A \in \Omega} \prod_{i \in A} \delta(g_i) \int_{\mathbb{R}^{|A|}_{< 0}} \int_{\mathbb{R}^{d - |A|}}  \prod_{j \in \mathcal{S} \setminus A} \delta(g_j - f_j)\mathbbm{1}_{f_j \geq 0} p(f_{1}, \dots, f_{d}) d\bm{f} \\
    &= \sum_{A \in \Omega} \prod_{i \in A} \delta(g_i) \int_{\mathbb{R}^{|A|}_{< 0}} p(\bm{f}_A, \bm{g}_{\mathcal{S}\setminus A}) d\bm{f}_A \prod_{j \in \mathcal{S} \setminus A}\mathbbm{1}_{g_j > 0} \\
    &= \sum_{A \in \Omega} \frac{1}{2^{|A|}} \prod_{i \in A} \delta(g_i) p( \bm{g}_{\mathcal{S}\setminus A}) \prod_{j \in \mathcal{S} \setminus A}\mathbbm{1}_{g_j > 0} ,
\end{align}
where in the second to last step we have used the well known property of the Delta function $\int_{-\infty}^{\infty} f(x)\delta(x - x_0)dx = f(x_0)$ and in the last step we used the fact that the density $p$ is symmetric around 0.
\end{proof}

\subsection{Proof of Theorem \ref{THM:DISTR_RELU_NETWORK_ANY_WIDTH}}
\label{sec:proof_thm_distr_relu_any_width}
\begin{tcolorbox}[boxrule=0pt, sharp corners]
\begin{theorem*}
Suppose $l \geq 2$, and the input has dimension $d_0$. Define the multi-index set $\mathcal{R} = [d_1] \times \dots \times [d_{l-1}]$ and introduce the vector $\bm{u}^{\bm{r}} \in \mathbb{R}^{l-1}$ through its components $\bm{u}^{\bm{r}}_i = \frac{1}{2}(r_i-d_l)$.
\begin{equation}
    p(\bm{f}_{\text{ReLU}}^{(l)}) = \sum_{\bm{r} \in \mathcal{R}}q_{\bm{r}}G^{l, 0}_{0,l}\left(\frac{||\bm{f}_{\text{ReLU}}^{(l)}||^2}{2^l  \sigma^{2}} \bigg \rvert 0, \bm{u}^{\bm{r}} \right) + q_0 \delta(\bm{f}_{\text{ReLU}}^{(l)}) ,
\end{equation}
where $\sigma^2 = \prod_{i=1}^l\sigma_i^2$ and the individual weights are given by
\begin{equation}
    q_{\bm{r}} =\pi^{-\frac{d_l}{2}}2^{-\frac{l}{2}d_l} (\sigma^2)^{-\frac{d_l}{2}}\prod_{i=1}^{l-1}\binom{d_{i}}{r_i} \frac{1}{2^{d_{i}} \Gamma\left( \frac{r_i}{2}\right)}  ,
\end{equation}
and 
\begin{equation}
    q_0 = 1 - \prod_{i=1}^{l-1}\frac{2^{d_i}-1}{2^{d_i}}
\end{equation}
\end{theorem*}
\end{tcolorbox}
\begin{proof}
Before starting the proof, note that there is a special case the has to be handled separately: the case in which \emph{all units} are inactive, i.e., the ReLU activation sets to zero all the pre-activation in a layer. This will be handled at end of the proof. First, let's assume that there is at least one active unit per layer.

The proof is again by induction. The base case ($l=2$) is stated in Lemma \ref{lemma:second_layer_relu}. For the general case, we use again an identical approach as in Theorem \ref{thm:distr_linear_network_any_width}. We expand the coefficients and write $c^l_{r_l} := \binom{d_l}{r_l}\frac{1}{2^{d_l}\Gamma\left( \frac{d_{l} - r_l}{2}\right)}$, where $l > 0$ is the layer index.
Induction step: assume that the pre-nonlinearities have the following form:
\begin{align}
p(f^{(l-1)}_1, \dots, f^{(l-1)}_{d_{l-1}}) &= \sum_{r_1=0}^{d_1-1} c^{1}_{r_1} \dots \sum_{r_{l-2}=0}^{d_{l-2}-1} c^{l-2}_{r_{l-2}}\pi^{-\frac{d_{l-1}}{2}}2^{-\frac{(l-1)d_{l-1}}{2}}(\tilde{\sigma}^2)^{-\frac{d_{l-1}}{2}} \\
& G^{l-1, 0}_{0,l-1}\left(\frac{||\bm{f}^{(l-1)}||^2}{2^{l-1}  \tilde{\sigma}^{2}} \bigg \rvert  0, \frac{1}{2}\left(d_1 - r_1 - d_{l-1}\right), \dots, \frac{1}{2}\left( d_{l-2} - r_{l-2} - d_{l-1}\right) \right) ,
\end{align}
where $\tilde{\sigma}^2 = \prod_{i=1}^{l-1}\sigma_i^2$. 
We know from Lemma \ref{lemma:relu_transform_density}, that the activations $g^{(l-1)}_1, \dots, g^{(l-1)}_{d_{l-1}}$ have the following density:
\begin{align}
    p_{\text{ReLU}}(g^{(l-1)}_1, \dots, g^{(l-1)}_{d_{l-1}}) &= \sum_{A \in \Omega} \frac{1}{2^{|A|}}\prod_{i \in A} \delta(g^{(l-1)}_i) p(\bm{g}^{(l-1)}_{\mathcal{S} \setminus A}) \prod_{j \in \mathcal{S} \setminus A} \mathbbm{1}_{g_j > 0} \\
    &= \sum_{A \in \Omega} \frac{1}{2^{|A|}} \sum_{r_1=0}^{d_1-1} c^1_{r_1} \dots \sum_{r_{l-2}=0}^{d_{l-2}-1} c^{l-2}_{r_{l-2}} \pi^{-\frac{d_{l-1}}{2}}2^{-\frac{(l-1)d_{l-1}}{2}}(\tilde{\sigma}^2)^{-\frac{d_{l-1}}{2}} \\
    & G^{l-1, 0}_{0,l-1}\left(\frac{||\bm{g}^{(l-1)}_{\mathcal{S} \setminus A}||^2}{2^{l-1}  \tilde{\sigma}^{2}} \bigg \rvert  0, \frac{1}{2}\left(d_1 - r_1 - d_A\right), \dots, \frac{1}{2}\left( d_{l-2} - r_{l-2} - d_A \right) \right) \\
    & \prod_{i \in A} \delta(g^{(l-1)}_i) \prod_{j \in \mathcal{S} \setminus A} \mathbbm{1}_{g_j > 0},
\end{align}
where $d_A := |\mathcal{S} \setminus A| = d_{l-1} - |A|$. Also, here we abuse the notation and consider that $\mathcal{S}$ is not in the power set, i.e., $\mathcal{S} \not \in \Omega$. This is to be consistent with the fact that we are handling the case in which at least one unit is active after the ReLU activation is applied. Now following a similar procedure as in Lemma \ref{lemma:second_layer_relu}, we have
\begin{align}
    p(f_1^{(l)}, \dots f_{d_l}^{(l)}) &= \int_{\mathbb{R}_{\geq 0}^{d_{l-1}}} p(f_1^{(l)}, \dots f_{d_l}^{(l)} | g_1^{(l-1)}, \dots g_{d_{1}}^{(l-1)} ) p_{\text{ReLU}}(g_1^{(l-1)}, \dots g_{d_1}^{(l-1)}) d\bm{g}^{(l-1)} ,
\end{align}
which is equal to:
\begin{align}
    & \int_{\mathbb{R}_{\geq 0}^{d_{l-1}}} \frac{1}{(2\pi\sigma_l^2 ||\bm{g}^{(l-1)}||^2)^{\frac{d_l}{2}}} \text{e}^{-\frac{||\bm{f}^{(l)}||^2}{2 \sigma_l^2||\bm{g}^{(l-1)}||^2}} \sum_{A \in \Omega} \frac{1}{2^{|A|}} \sum_{r_1=0}^{d_1-1} c^1_{r_1} \dots \sum_{r_{l-2}=0}^{d_{l-2}-1} c^{l-2}_{r_{l-2}}\\ 
    & \pi^{-\frac{d_{l-1} - |A|}{2}}2^{-\frac{(l-1)(d_{l-1}-|A|)}{2}}(\tilde{\sigma}^2)^{-\frac{d_{l-1}-|A|}{2}} \\
    & G^{l-1, 0}_{0,l-1}\left(\frac{||\bm{g}^{(l-1)}_{\mathcal{S} \setminus A}||^2}{2^{l-1}  \tilde{\sigma}^{2}} \bigg \rvert  0, \frac{1}{2}\left(d_1 - r_1 - d_A\right), \dots, \frac{1}{2}\left( d_{l-2} - r_{l-2} - d_A \right) \right) \\
    & \prod_{i \in A} \delta(g^{(l-1)}_i) \prod_{j \in \mathcal{S} \setminus A} \mathbbm{1}_{g_j > 0} d\bm{g}^{(l-1)} \\
    &= \sum_{A \in \Omega} \frac{\pi^{-\frac{d_{l-1} - |A|}{2}}2^{-\frac{(l-1)(d_{l-1}-|A|)}{2}}(\tilde{\sigma}^2)^{-\frac{d_{l-1}-|A|}{2}}}{2^{|A|} (2\pi\sigma_l^2)^{\frac{d_l}{2}}}\sum_{r_1=0}^{d_1-1} c^1_{r_1} \dots \sum_{r_{l-2}=0}^{d_{l-2}-1} c^{l-2}_{r_{l-2}} \\
    & \int_{\mathbb{R}_{> 0}^{d_{l-1}}} \frac{1}{( ||\bm{g}_{\mathcal{S}\setminus A}^{(l-1)}||^2)^{\frac{d_l}{2}}} \text{e}^{-\frac{||\bm{f}^{(l)}||^2}{2 \sigma_l^2||\bm{g}_{\mathcal{S}\setminus A}^{(l-1)}||^2}} \\
    & G^{l-1, 0}_{0,l-1}\left(\frac{||\bm{g}^{(l-1)}_{\mathcal{S} \setminus A}||^2}{2^{l-1}  \tilde{\sigma}^{2}} \bigg \rvert  0, \frac{1}{2}\left(d_1 - r_1 - d_A\right), \dots, \frac{1}{2}\left( d_{l-2} - r_{l-2} - d_A \right) \right) d\bm{g}_{\mathcal{S} \setminus A}^{(l-1)} 
\end{align}
For each set $A$ we have a $(d_{l-1} - |A|)$- dimensional integral that can be solved using once again Lemma \ref{lemma:integral_for_induction} \footnote{see proof of Lemma \ref{lemma:second_layer_relu} for a small but important detail of this integral}. Note that for the new Meijer-G coefficients of Lemma \ref{lemma:integral_for_induction}:
\begin{equation}
    \frac{1}{2}\left(d_i - r_i - d_{l-1} + |A| \right) + \frac{1}{2}\left(d_{l-1} - |A| - d_l \right) = \frac{1}{2}\left(d_i - r_i - d_l \right)
\end{equation}
holds for all $i \in [d_{l-2}]$. Therefore the solution of each integral is equal to 
\begin{align}
    & \frac{1}{2}\tilde{C}_A 2^{\frac{1}{2}(d_{l-1} - |A| - d_l)(l-1)} \tilde{\sigma}^{(d_{l-1} - |A| - d_l)} \\ & G^{l, 0}_{0,l}\left(\frac{||\bm{f}^{(l)}||^2}{2^l  \tilde{\sigma}^{2}} \bigg \rvert  0, \frac{1}{2}\left(d_1 - r_1 - d_l\right), \dots, \frac{1}{2}\left( d_{l-1} - r_{l-1} - d_l\right) \right).
\end{align}
The new coefficient for every set $A$ is:
\begin{align}
    c_A &= \frac{\pi^{-\frac{d_{l-1} - |A|}{2}}2^{-\frac{(l-1)(d_{l-1}-|A|)}{2}}(\tilde{\sigma}^2)^{-\frac{d_{l-1}-|A|}{2}}}{2^{|A|} (2\pi\sigma_l^2)^{\frac{d_l}{2}}} \frac{1}{2}\tilde{C}_A 2^{\frac{1}{2}(d_{l-1} - |A| - d_l)(l-1)} \tilde{\sigma}^{(d_{l-1} - |A| - d_l)} \\
    &= \frac{\cancel{\pi^{-\frac{d_{l-1}-|A|}{2}}}\cancel{2^{-\frac{(l-1)(d_{l-1}-|A|)}{2}}(\tilde{\sigma}^2)^{-\frac{d_{l-1}-|A|}{2}}}}{\cancel{2^{|A|}} (2\pi\sigma_l^2)^{\frac{d_l}{2}}} \frac{1}{\cancel{2}}\frac{\cancel{\pi^{\frac{d_{l-1} - |A|}{2}}}}{2^{d_{l-1} - \cancel{|A|} - \cancel{1}} \Gamma\left( \frac{d_{l-1} - |A|}{2}\right)} \\
    & \tilde{C}_A 2^{\frac{1}{2}(\cancel{d_{l-1} - |A|} - d_l)(l-1)} \tilde{\sigma}^{(\cancel{d_{l-1} - |A|} - d_l)}  \\
    &= \frac{\pi^{-\frac{d_l}{2}}2^{-\frac{ld_l}{2}}(\sigma^2)^{-\frac{d_l}{2}}}{2^{d_{l-1}}\Gamma\left( \frac{d_{l-1} - |A|}{2}\right)}
\end{align} 
Therefore, because the dependence on $A$ is only through its cardinality $r^{l-1}_{l-1} := |A|$, we define:
\begin{equation}
    c^{l-1}_{r_{l-1}} :=  \binom{d_{l-1}}{r_{l-1}} \frac{1}{2^{d_{l-1}} \Gamma\left( \frac{d_{l-1} - r_{l-1}}{2}\right)}
\end{equation}
So the solution is:
\begin{align}
    p(f^{(l)}_1, \dots, f^{(l)}_{d_l}) &= \sum_{r_1=0}^{d_1-1} c^1_{r_1} \dots \sum_{r_{l-1}=0}^{d_{l-1}-1} c^{l-1}_{r_{l-1}} \pi^{-\frac{d_l}{2}}2^{-\frac{ld_l}{2}}(\sigma^2)^{-\frac{d_l}{2}} \\ 
    & G^{l, 0}_{0,l}\left(\frac{||\bm{f}^{(l)}||^2}{2^l  \sigma^{l}} \bigg \rvert  0, \frac{1}{2}\left(d_1 - r_1 - d_l\right), \dots, \frac{1}{2}\left( d_{l-1} - r_{l-1} - d_l\right) \right) .
\end{align}

The final form of this equation stated in the theorem is obtained by grouping all the coefficients not involving the Meijer-G function, and substituting $r_i \leftarrow d_i - r_i$ and use the property $\binom{d_i}{r_i} = \binom{d_i}{d_i-r_i}$.

\paragraph{Special case: all units are inactive} If at least in one layer it happens that all post-activations are zero, then the distribution of the network is a point mass at 0. Let's call this event $E$, and its probability $q_0$. The probability of its complement $\bar{E}$ is the probability that for all the intermediate layers, at least one unit is active. These are $l-1$ independent events, the probability of each being $\frac{2^{d_i} - 1}{2^{d_i}}$ (one unit is active in $2^{d_i} -1$ cases out of the all possible combinations of units). Therefore we can conclude that:
\begin{align}
    q_0 = 1 - \prod_{i=1}^{l-1}\frac{2^{d_i - 1}}{2^{d_i}}
\end{align}
\end{proof}

\subsection{Base case for ReLU nets}
\begin{tcolorbox}[boxrule=0pt, sharp corners]
\begin{lemma}[second layer pre-activations density]
\label{lemma:second_layer_relu}
Let $\sigma = \sigma_1 \sigma_2$.
Conditioned on the event that at least one unit of the first layer is active ($g_j^{(1)} \neq 0$ for at least one $j \in [d_1]$), the density of the second layer's pre-activations $f_1^{(2)}, \dots f_{d_2}^{(2)}$ is the following linear combination of Meijer-G functions:
\begin{equation}
    p(\textbf{f}_1^{(2)}) = \sum_{r=0}^{d_1-1}c_r 2^{-d_2}(\sigma^2)^{-\frac{d_2}{2}}\pi^{\frac{-d_2}{2}} G^{2,0}_{0,2}\left(\frac{||\bm{f}^{(2)} ||^2}{4\sigma^2} \bigg \rvert 0, \frac{1}{2}((d_1 - r ) - d_2)\right) + q_0 \delta(\textbf{f}_1^{(2)}) ,
\end{equation}
where $c_r :=\binom{d_1}{r} \frac{1}{2^{d_1} \Gamma\left( \frac{d_1 - r}{2}\right)}$ and $q_0 := 1 - \frac{2^{d_1-1}}{2^{d_i}}$.
\end{lemma}
\end{tcolorbox}
\begin{proof}
\begin{align}
    p(f_1^{(2)}, \dots f_{d_2}^{(2)}) &= \int_{\mathbb{R}_{\geq 0}^{d_1}} p(f_1^{(2)}, \dots f_{d_2}^{(2)} | g_1^{(1)}, \dots g_{d_1}^{(1)} ) p_{\text{ReLU}}\left(g_1^{(1)}, \dots g_{d_1}^{(1)}\right) d\bm{g}^{(1)}\\
    &= \int_{\mathbb{R}_{\geq 0}^{d_1}} \prod_{k=1}^{d_2}p(f_k^{(2)} | \bm{g}^{(1)}) \prod_{k'=1}^{d_1}p_{ReLU}\left(g_{k'}^{(1)}\right) d\bm{g}^{(1)}\\
    &= \int_{\mathbb{R}_{\geq 0}^{d_1}} \frac{1}{(2\pi\sigma_2^2 ||\bm{g}^{(1)}||^2)^{\frac{d_2}{2}}} \exp \left({-\frac{||\bm{f}^{(2)}||^2}{2 \sigma_2^2||\bm{g}^{(1)}||^2}}\right) \cdot \\
    & \sum_{A \in \Omega} \frac{1}{2^{|A|}} \prod_{i \in A} \delta(g_i^{(1)}) \frac{1}{(2 \pi \sigma_1^2)^{\frac{d_1 - |A|}{2}}} \exp\left(- \frac{\sum_{i \in \mathcal{S} \setminus A} (g_i^{(1)})^2}{2 \sigma_1^2} \right) \prod_{i \in \mathcal{S} \setminus A} \mathbbm{1}_{g_i^{(1)} > 0} d\bm{g}^{(1)} \\
    &= \sum_{A \in \Omega} \frac{1}{2^{|A|}}\frac{1}{(2 \pi \sigma_1^2)^{\frac{d_1 - |A|}{2}} (2\sigma_2^2)^{\frac{d_2}{2}}} \int_{\mathbb{R}_{\geq 0}^{d_1}} \frac{1}{( ||\bm{g}^{(1)}||^2)^{\frac{d_2}{2}}} \exp \left({-\frac{||\bm{f}^{(2)}||^2}{2 \sigma_2^2||\bm{g}^{(1)}||^2}}\right) \\
    & \prod_{i \in A} \delta(g_i^{(1)}) \exp\left(- \frac{\sum_{i \in \mathcal{S} \setminus A} (g_i^{(1)})^2}{2 \sigma_1^2} \right) \prod_{i \in \mathcal{S} \setminus A} \mathbbm{1}_{g_i^{(1)} > 0} d\bm{g}^{(1)} ,
\end{align}
where we can exchange sum and integration due to non-negativeness of the integration variables (Tonelli's theorem). Also, here we abuse the notation and consider that $\mathcal{S}$ is not in the power set, i.e., $\mathcal{S} \not \in \Omega$. This is to be consistent with the fact that we are conditioning on the event in which at least one unit is active after the ReLU activation is applied in the first layer. Now we can use the property of the delta function $\int f(x)\delta(x - x_0)dx = f(x_0)$ and the property of the indicator function $\int_A f(x)\mathbbm{1}_{x \in B} dx = \int_B f(x) dx$ and get:
\begin{align}
    \sum_{A \in \Omega} & \frac{1}{2^{|A|}}\frac{1}{(2 \pi \sigma_1^2)^{\frac{d_1 - |A|}{2}} (2\sigma_2^2)^{\frac{d_2}{2}}} \\
    & \int_{\mathbb{R}_{> 0}^{d_1}} \frac{1}{( ||\bm{g}^{(1)}_{\mathcal{S}\setminus A}||^2)^{\frac{d_2}{2}}} \exp \left({-\frac{||\bm{f}^{(2)}||^2}{2 \sigma_2^2||\bm{g}^{(1)}_{\mathcal{S}\setminus A}||^2}}\right) \exp\left(- \frac{||\bm{g}^{(1)}_{\mathcal{S}\setminus A}||^2}{2 \sigma_1^2} \right)  d\bm{g}^{(1)}_{\mathcal{S}\setminus A}.
\end{align}
Note that the above integral is $(d_1 - |A|)$ dimensional due to the effect of the delta. Now the integral(s) above can be solved in an equivalent manner as in the previous section using Lemma \ref{lemma:integral_for_induction}\footnote{Note that the integral is only for the positive reals. Lemma \ref{lemma:integral_for_induction} can still be used because when switching to spherical coordinates, we are interested in the radius part, while the angular constant can still be calculated, but now we the angles are all from 0 to $\frac{\pi}{2}$}, and they are equal to 
\begin{align}
    & \frac{1}{2}\tilde{C}_{A}(2 \sigma_1^2)^{\frac{1}{2}(d_1 - |A| - d_2)} G^{2,0}_{0,2}\left(\frac{||\bm{f}^{(2)} ||^2}{4\sigma^2} \bigg \rvert 0, \frac{1}{2}((d_1 - |A|) - d_2)\right) 
\end{align}
So we can conclude that:
\begin{align}
   p(f_1^{(2)}, \dots f_{d_2}^{(2)}) &= \sum_{A \in \Omega} \frac{1}{2^{|A|}}\frac{1}{2}\tilde{C}_{A}(2 \sigma_1^2)^{\frac{1}{2}(d_1 - |A| - d_2)}\frac{1}{(2 \pi \sigma_1^2)^{\frac{d_1 - |A|}{2}} (2\sigma_2^2)^{\frac{d_2}{2}}} \\ & G^{2,0}_{0,2}\left(\frac{||\bm{f}^{(2)} ||^2}{4\sigma^2} \bigg \rvert 0, \frac{1}{2}((d_1 - |A|) - d_2)\right) \\
   &= \pi^{-\frac{d_1}{2}}\sum_{r=0}^{d_1-1}\binom{d_1}{r} \frac{\tilde{C}_{r}}{2^{r + 1} (2\sigma_1^2)^\frac{d_2}{2}(2\sigma_2^2)^\frac{d_2}{2} \pi^{\frac{- r + d_2}{2}}} \\ & G^{2,0}_{0,2}\left(\frac{||\bm{f}^{(2)} ||^2}{4\sigma^2} \bigg \rvert 0, \frac{1}{2}((d_1 - r) - d_2)\right) ,
\end{align}
where we have used the fact that the expression depends on the set $A \in \Omega$ only through $|A|$, and therefore we can use the fact that the number of subsets with $r$ elements is given by the binomial coefficient $\binom{d_1}{r}$. 
Define:
\begin{align}
    c_r &:= \binom{d_1}{r} \pi^{-\frac{d_1}{2}} \frac{\tilde{C}_{r}}{2^{r + 1}} \pi^{\frac{r}{2}} \\
    &= \binom{d_1}{r} \pi^{-\frac{d_1}{2}}\frac{\pi^{\frac{d_1 - r}{2}}}{2^{d_1 - r - 1} \Gamma\left( \frac{d_1 - r}{2}\right)} \frac{1}{2^{r + 1}} \pi^{\frac{r}{2}} \\
    &= \binom{d_1}{r} \frac{1}{2^{d_1} \Gamma\left( \frac{d_1 - r}{2}\right)} .
\end{align}
So we can conclude:
\begin{equation}
    p(f_1^{(2)}, \dots f_{d_2}^{(2)}) = \sum_{r=0}^{d_1-1}c_r 2^{-d_2}(\sigma^2)^\frac{-d_2}{2}\pi^{\frac{-d_2}{2}} G^{2,0}_{0,2}\left(\frac{||\bm{f}^{(2)} ||^2}{4\sigma^2} \bigg \rvert 0, \frac{1}{2}((d_1 - r ) - d_2)\right) .
\end{equation}

Finally, there is the special case where all the units are inactive (set to zero). This happens with probability $q_0 = \frac{1-2^{d_i-1}}{2^{d_i}}$. 
\end{proof}
\paragraph{Remark}
Any non empty subset of $d < d_2$ units has the same distribution (with terms involving $d_2$ replaced by $d$).

\subsection{Resulting moments}
\label{sec:proofs_resulting_moments}
Let $d_l = 1$, $d_1, \dots, d_{l-1} = m$, and $b_i = \frac{1}{2}(r_i - 1)$, $i = 1, \dots, l-1$. 
\begin{equation}
    \begin{split}
        \mathbb{E}\left[Z^{2k}\right] &= \int_{\mathbb{R}} z^{2k}\sum_{r_1=1}^{m} c^1_{r_1} \dots \sum_{r_{l-1}=1}^{m} c^{l-1}_{r_{l-1}} \pi^{-\frac{1}{2}}2^{\frac{l}{2}}(\sigma^2)^{-\frac{1}{2}}G^{l, 0}_{0,l}\left(\frac{z^2}{2^l  \sigma^2} \bigg \rvert  0, b_1, \dots, b_{l-1} \right)dz \\
        &=  \sum_{r_1=1}^{m} c^1_{r_1} \dots \sum_{r_{l-1}=1}^{m} c^{l-1}_{r_{l-1}} \pi^{-\frac{1}{2}}2^{\frac{l}{2}}(\sigma^2)^{-\frac{1}{2}}  \int_{0}^{\infty}r^{2k}G^{l, 0}_{0,l}\left(\frac{r^2}{2^l  \sigma^2} \bigg \rvert  0, b_1, \dots, b_{l-1} \right)dr \\
        &= \sum_{r_1=1}^{m} c^1_{r_1} \dots \sum_{r_{l-1}=1}^{m} c^{l-1}_{r_{l-1}} \pi^{-\frac{1}{2}}2^{\frac{l}{2}}(\sigma^2)^{-\frac{1}{2}}  \int_{0}^{\infty}x^{k - \frac{1}{2}}G^{l, 0}_{0,l}\left(\frac{x}{2^l  \sigma^2} \bigg \rvert  0, b_1, \dots, b_{l-1} \right)dx \\
        &= \sum_{r_1=1}^{m} c^1_{r_1} \dots \sum_{r_{l-1}=1}^{m} c^{l-1}_{r_{l-1}} \pi^{-\frac{1}{2}}2^{\frac{l}{2}}(\sigma^2)^{-\frac{1}{2}} \left(\frac{1}{2^l \sigma^2}\right)^{-k -\frac{1}{2}}\Gamma\left(k + \frac{1}{2}\right)\prod_{i=1}^{l-1}\Gamma\left(k + \frac{1}{2} + b_i\right) \\
        &= \pi^{-\frac{1}{2}} 2^{\frac{l}{2}}(\sigma^2)^{-\frac{1}{2}} \left(\frac{1}{2^l \sigma^2}\right)^{-k -\frac{1}{2}}\Gamma\left(k + \frac{1}{2}\right) \sum_{r_1=1}^{m} c^1_{r_1} \dots \sum_{r_{l-1}=1}^{m} c^{l-1}_{r_{l-1}} \prod_{i=1}^{l-1}\Gamma\left(k + \frac{1}{2} + b_i\right) \\
        &= (2k-1)!! 2^{k(l-1)}\sigma^{2k}\left(\frac{1}{2^m}\sum_{r=1}^{m}\binom{m}{r} \frac{\Gamma\left(k+\frac{r}{2}\right)}{\Gamma\left(\frac{r}{2}\right)}\right)^{l-1} .
    \end{split} 
\end{equation}

For instance, for the variance ($k=1$) the sum becomes:
\begin{align}
    &\sum_{r_1=1}^{m} c^1_{r_1} \dots \sum_{r_{l-1}=1}^{m} \binom{m}{r_{l-1}} \frac{1}{2^{m} \Gamma\left( \frac{r_{l-1}}{2}\right)}  \prod_{i=1}^{l-1}\frac{1}{2}r_{i} \Gamma\left(\frac{1}{2}r_i\right) \\
    &= \sum_{r_1=1}^{m} \binom{m}{r_{1}} \frac{r_{1}}{2^{m+1}} \dots \sum_{r_{l-1}=1}^{m} \binom{m}{r_{l-1}} \frac{r_{l-1}}{2^{m+1}}  .
\end{align}
Now each sum can be solved independently:
\begin{align}
    \sum_{r_{i}=1}^{m} \binom{m}{r_{i}} \frac{ r_{i}}{2^{m+1}} &= \frac{1}{2^{m+1}}\left[ \sum_{r_{i}=1}^{m} \binom{m}{r_{i}}r_i \right] \\
    &= \frac{1}{2^{m+1}}\left[m2^{m-1} \right] \\
    &= \frac{m}{4} .
\end{align}
Therefore the variance is:
\begin{align}
    \mathbb{V}[Z] &=  \pi^{-\frac{1}{2}} 2^{\frac{l}{2}}(\sigma^2)^{-\frac{1}{2}} \left(\frac{1}{2^l \sigma^2}\right)^{-1 -\frac{1}{2}}\Gamma\left(1 + \frac{1}{2}\right) \frac{m^{l-1}}{2^{2(l-1)}} \\
    &=\frac{1}{2} 2^{\frac{l}{2}}(\sigma^2)^{-\frac{1}{2}} \left(\frac{1}{2^{-\frac{3}{2}l} (\sigma^{2})^{-\frac{3}{2}}}\right) \frac{m^{l-1}}{2^{2(l-1)}} \\
    &= \frac{1}{2} 2^{l} \sigma^2 \frac{m^{l-1}}{2^{2(l-1)}} \\
    &= \frac{\sigma^2 m^{l-1}}{2^{l-1}} .
\end{align}

Note how the variance of a ReLU net is significantly reduced if compared with the variance of a linear network of the same depth (compare with Eq. \ref{eq:variance_linear_net}). 
Similarly, one can get the fourth moment:
\begin{equation}
    \mathbb{E}[Z^4] = \frac{3 (\sigma^2)^{2}(m+5)^{l-1}m^{l-1}}{2^{2(l-1)}} .
\end{equation}

Therefore the kurtosis is:
\begin{equation}
    \kappa = 3 \left(\frac{m+5}{m}\right)^{l-1} .
\end{equation}

Note how ReLU nets are more heavy-tailed than linear nets.

To calculate the asymptotic moments we need three technical Lemmas that express the quantities encountered in a better form. First we describe the coefficients of a factorized polynomial:
\begin{lemma}
\label{coef_poly}
Consider coefficients $a_{1}, \dots, a_m \in \mathbb{R}$. Define the polynomial
\begin{equation}
    p(x) = \prod_{i=1}^{m}(x+a_i) = \sum_{i=1}^{m}\alpha_i x^{i} .
\end{equation}

Then it holds that $\alpha_m = 1$ and $\alpha_{m-1} = \sum_{i=1}^{m}a_i$.
\end{lemma}Next we use Lemma~\ref{coef_poly} to write the ratio of Gamma functions as a polynomial:
\begin{lemma}
\label{ratio-of-gammas}
Fix $k \in \mathbb{N}$ and $x \in \mathbb{R}$. Then we can express the fraction of Gamma functions as follows:
\begin{equation}
    \frac{\Gamma(k + \frac{x}{2})}{\Gamma(\frac{x}{2})} = P_{k}(x) = \sum_{i=0}^{k}\alpha_i x^{i} ,
\end{equation}

where $P_k$ is a $k$-th order polynomial with  coefficients $\alpha_k = 2^{-k}$ and $\alpha_{k-1} = \frac{k^2-k}{2^k}$.
\end{lemma}
\begin{proof}
The leading coefficient can easily be obtained from multiplying together the terms $\frac{m}{2}$. From Lemma~\ref{coef_poly} we conclude that 
\begin{equation}
    \alpha_{k-1} = \frac{1}{2^{k-1}}\sum_{i=1}^{k}(k-i) = \frac{1}{2^{k-1}}\left(k^2 - \frac{k(k+1)}{2}\right) =  \frac{k^2-k}{2^k} .
\end{equation}

\end{proof}
Next we need to control the sums involving the factorials. Since we just expressed the ratio of Gamma functions as a polynomial, we essentially need to know how to control sums of the type
\begin{equation}
    \frac{1}{2^m}\sum_{r=1}^{m}\binom{m}{r}r^k  ,
\end{equation}

which amounts to controlling the moments of a binomial distribution with fault probability $p=\frac{1}{2}$. We do this as follows:
\begin{lemma}
\label{binomial-moments}
Fix $k,m \in \mathbb{N}$. Then we can express the following sum as a polynomial $\forall k \in \mathbb{N}$:
\begin{equation}
    \frac{1}{2^m}\sum_{r=0}^{m} \binom{m}{r}r^k = \frac{m}{2^{k}}Q_{k-1}(m) .
\end{equation}

where $Q_{k-1}$ is a $k-1$-th order polynomial. Moreover, writing $Q_l$ in monomial basis
\begin{equation}
    Q_{l}(m) = \sum_{i=0}^{l}\alpha_i m^{i}  ,
\end{equation}

it holds that $\alpha_{l} = 1$ and $\alpha_{l-1} = \frac{l(l+1)}{2}$ $\forall l \in \mathbb{N}$. 
\end{lemma}
\begin{proof}
For a proof of the recursion, we refer to \citet{boros_moll_2004, binomial_moments}. Moreover the polynomials satisfy the recursion
\begin{equation}
    Q_{k}(m) = 2mQ_{k-1}(m) - (m-1)Q_{k-1}(m-1)  .
\end{equation}

Denote by $\alpha^{(k)}$ the coefficients of $Q_k$, so $\alpha^{(k)}_0, \dots, \alpha^{(k)}_k$. Notice that the leading coefficient of $Q_k$ is thus $\alpha^{(k)}_k$ and for $Q_{k-1}$ it is $\alpha^{(k-1)}_{k-1}$. Using the recursion and performing a comparison of coefficients we see that
\begin{equation}
    \alpha^{(k)}_k = 2 \alpha^{(k-1)}_{k-1} - \alpha^{(k-1)}_{k-1} = \alpha^{(k-1)}_{k-1} .
\end{equation}

Using the fact that for $k=1$
\begin{equation}
    \frac{1}{2^m}\sum_{r=0}^{m} \binom{m}{r}r = \frac{m}{2} = \frac{m}{2^{1}}Q_{0}(m) ,
\end{equation}

we conclude that $\alpha_0^{0} = 1$ and thus $\alpha_k = 1$ $\forall k \in \mathbb{N}$. For the second coefficient, namely $\alpha_{k-1}^{(k)}$ for $Q_k$ and $\alpha_{k-2}^{(k-1)}$ for $Q_{k-1}$, we will again use the recursion. Let us first again express the polynomials in monomial bases, i.e.\ 
\begin{equation}
    Q_{k}(m) = \sum_{i=0}^{k}\alpha_i^{(k)} m^{i} \hspace{2mm}, \hspace{5mm} Q_{k-1}(m) = \sum_{i=0}^{k-1}\alpha_i^{(k-1)} m^{i}
\end{equation}

Using the recursion we thus see that
\begin{equation}
    \sum_{i=0}^{k}\alpha_i^{(k)} m^{i} = \sum_{i=0}^{k-1}2\alpha_i^{(k-1)} m^{i+1} - \sum_{i=0}^{k-1}\alpha_i^{(k-1)} (m-1)^{i+1} .
\end{equation}

We have to understand the terms involving $m^{k-1}$. Thus we need to expand $(m-1)^{k}$ which we can do with the help of Lemma \ref{coef_poly}:
\begin{equation}
    (m-1)^{k} = m^k - km^{k-1} + \dots
\end{equation}
We also need to expand the next polynomial as follows:
\begin{equation}
    (m-1)^{k-1} = m^{k-1} + \dots
\end{equation}

Collecting all the coefficients, we end up with the following recursion for the second coefficient:
\begin{equation}
    \begin{split}
        \alpha_{k-1}^{(k)} &= 2 \alpha_{k-2}^{(k-1)} - \alpha_{k-2}^{(k-1)} + k\alpha_{k-1}^{(k-1)} \\
        &= \alpha_{k-2}^{(k-1)} + k .
    \end{split}
\end{equation}
Using the fact that
\begin{equation}
    Q_1(m) = m + 1 .
\end{equation}

Thus $\alpha_{0}^{(1)} = 1$, we conclude that
\begin{equation}
    \alpha_{k-1}^{(k)} = 1 + \sum_{i=2}^{k}i =  \frac{k(k+1)}{2} .
\end{equation}

\end{proof}
Finally, we need a result on exponential functions and their limit definition:
\begin{lemma}
\label{exp-limit}
Fix $c \in \mathbb{R}$ and $\gamma \in \mathbb{R}_{+}$. Then we have the following limit:
\begin{equation}
   \lim_{m \xrightarrow[]{}\infty}\left(1 + \frac{c}{m} + \mathcal{O}\left(\frac{1}{m^2}\right)\right)^{(\gamma m)} = e^{\gamma c} .
\end{equation}

Moreover, it holds that 
\begin{equation}
    \lim_{m \xrightarrow[]{}\infty}\left(1 + \frac{c}{m}+ \mathcal{O}\left(\frac{1}{m^2}\right)\right)^{m^{\beta}} = \begin{cases}\infty \hspace{5mm} \text{if } \beta>1 \\
    1 \hspace{7mm} \text{if }\beta <1\end{cases} 
\end{equation}
\end{lemma}
\begin{proof}
This can be found in standard analysis books such as \citet{rudin}.
\end{proof}
\subsection{Proof of Theorem \ref{relu_limit}}
We can now prove the convergence of the moments as follows.

\begin{tcolorbox}
\begin{theorem*}
Consider the distribution of the output $p\left(f_{\text{ReLU}}^{(L)}\right)$, as defined in Thm.~\ref{THM:DISTR_RELU_NETWORK_ANY_WIDTH}. Denote  $X \sim \mathcal{N}(0,1)$, $Y \sim \mathcal{LN}(-\frac{5}{4}\gamma, \frac{5}{4}\gamma)$ for $X \perp Y$ and the resulting normal log-normal mixture by $Z=XY$, for $\gamma >0$. Let the depth grow as $L = c + \gamma m^{\beta}$ where $\beta \geq 0$ and $c \in \mathbb{N}$ fixed. Then it holds that for $k>1$
\begin{equation}
    \mathbb{E}\left[\left(f^{(L)}_{\text{ReLU}}\right)^{2k}\right] \xrightarrow[]{m \xrightarrow[]{}\infty} \begin{cases}
\mathbb{E}[X^{2k}] = (2k-1)!! \hspace{18mm} \text{if } \beta < 1\\[2mm]
\mathbb{E}[Z^{2k}] = e^{\frac{5}{2}\gamma k(k-1)}(2k-1)!! \hspace{3mm} \text{ if } \beta = 1\\[2mm]
\infty \hspace{44mm} \text{if } \beta>1
\end{cases}
\end{equation}
where $(2k-1)!! = (2k-1) \dots 3 \cdot 1$ denotes the double factorial (by symmetry, odd moments are zero). Moreover, for $\beta < 1$ it holds that 
\begin{equation}
    p(f^{(L)}_{\text{ReLU}}) \xrightarrow[]{d} X \hspace{4mm} \text{for } m \xrightarrow[]{} \infty
\end{equation}

\end{theorem*}
\end{tcolorbox}
\begin{proof}
Recall that we arrived at 
\begin{equation}
    \mathbb{E}\left[\left(f^{(l)}_{\text{ReLU}}\right)^{2k}\right] = (2k-1)!! 2^{k(l-1)}\sigma_{w}^{2kl}\left(\frac{1}{2^m}\sum_{r=1}^{m}\binom{m}{r} \frac{\Gamma\left(k+\frac{r}{2}\right)}{\Gamma\left(\frac{r}{2}\right)}\right)^{l-1} .
\end{equation}

Using the NTK parametrization for ReLU, i.e.\ $\sigma_1^2 = 1$ and $\sigma_2^2 = \dots = \sigma_l^2 = \frac{2}{m}$, this amounts to
\begin{equation}
    \mathbb{E}\left[\left(f^{(l)}_{\text{ReLU}}\right)^{2k}\right] = (2k-1)!! \left(\frac{2^{2k}}{m^{k}2^m}\sum_{r=1}^{m}\binom{m}{r} \frac{\Gamma\left(k+\frac{r}{2}\right)}{\Gamma\left(\frac{r}{2}\right)}\right)^{l-1} .
\end{equation}
We thus essentially need to understand the term
\begin{equation}
    M(m) = \frac{2^{2k}}{2^mm^{k}}\sum_{r=1}^{m}\binom{m}{r} \frac{\Gamma\left(k+\frac{r}{2}\right)}{\Gamma\left(\frac{r}{2}\right)} 
\end{equation}

We first use Lemma~\ref{ratio-of-gammas} to expand the ratio $\frac{\Gamma\left(k+\frac{r}{2}\right)}{\Gamma\left(\frac{r}{2}\right)}$ as a polynomial. Denote the coefficients by $\beta_i$ for $i=1, \dots, k$ ($i\not=0$ because the polynomial has no intercept). We then swap the two sums:
\begin{equation}
    \begin{split}
        M(m) &= \frac{2^{2k}}{2^mm^{k}}\sum_{r=1}^{m}\binom{m}{r} \sum_{i=1}^{k}r^{i}\beta_i = \frac{2^{2k}}{m^{k}} \sum_{i=1}^{k}\beta_i \frac{1}{2^m}\sum_{r=1}^{m}\binom{m}{r}r^{i} .
    \end{split}
\end{equation}
Now we can apply Lemma~\ref{binomial-moments} to expand the inner sum for each $i$, denoting the corresponding polynomials again by $Q_i$: 
\begin{equation}
    \begin{split}
        M(m) &= \frac{2^{2k}}{m^{k}} \sum_{i=1}^{k}\beta_i \frac{m}{2^{i}}Q_{i-1}(m) .
    \end{split}
\end{equation}
Notice that $mQ_{i-1}(m)$ is a polynomial of order $i$. For large $m$, the factor $\frac{1}{m^k}$ dominates all such polynomials except for the one with $i=k$. Thus in the large-width limit it holds
\begin{equation}
    M(m) \xrightarrow[]{m \xrightarrow[]{}\infty} 2^{2k}\beta_{k}\frac{1}{2^{k}} = 1 .
\end{equation}
where we used that the leading coefficient of $Q_{k-1}$ is $1$. For fixed depth $l \in \mathbb{N}$ or depth growing as $l=m^{\beta}$ for $\beta <1$, we can pull the limit $\lim_{m \xrightarrow[]{}\infty}M(m)^{l-1}$ inside and conclude that 
\begin{equation}
    \mathbb{E}\left[\left(f^{(l)}_{\text{ReLU}}\right)^{2k}\right] \xrightarrow[]{m \xrightarrow[]{}\infty} (2k-1)!! .
\end{equation}
If $\beta > 1$, we obtain a divergence of the even moments for $k>1$ to infinity as $m$ increases as the exponent grows faster than $m$. Note that for $k=1$ however, we have that $\mathbb{E}\left[\left(f^{(l)}_{\text{ReLU}}\right)^{2}\right]=1$ also in the $\beta > 1$ limit. \\
For depth growing as  $l-1 = \gamma m$ (so $\beta=1)$, we have to be a bit more careful since we need to compute the coefficient in front of $\frac{1}{m}$, similarly as in the linear case. We now need to collect all the polynomial terms in $M(m)$ giving rise to a $\frac{1}{m}$ factor. First recall that
\begin{equation}
    \begin{split}
        M(m) &= \frac{2^{2k}}{m^{k}} \sum_{i=1}^{k}\beta_i \frac{m}{2^{i}}Q_{i-1}(m) .
    \end{split}
\end{equation}
The only coefficients contributing to $\frac{1}{m}$ are the second highest coefficient of $Q_{k-1}$ and the highest coefficient of $Q_{k-2}$. Using Lemma \ref{binomial-moments} and Lemma \ref{ratio-of-gammas}, we hence find that 
\begin{equation}
    \begin{split}
        M(m) &= 1 + 2^{2k}\left(\beta_k \frac{k(k-1)}{2}\frac{1}{2^k} + \beta_{k-1}\frac{1}{2^{k-1}}\right)\frac{1}{m} + \mathcal{O}\left(\frac{1}{m^2}\right) \\
        &= 1 + 2^{2k}\left(\frac{1}{2^{2k+1}}(k-1)k + \frac{1}{2^{2k-1}}(k-1)k\right)\frac{1}{m} + \mathcal{O}\left(\frac{1}{m^2}\right) \\
        &= 1 + \left(\frac{5((k-1)}{2}\right)\frac{1}{m} + \mathcal{O}\left(\frac{1}{m^2}\right) .
    \end{split}
\end{equation}
Applying \ref{exp-limit} concludes that
\begin{equation}
    \begin{split}
        \mathbb{E}\left[\left(f^{(l)}_{\text{ReLU}}\right)^{2k}\right] &= (2k-1)!! \left(1 + \left(\frac{5k(k-1)}{2}\right)\frac{1}{m} + \mathcal{O}\left(\frac{1}{m^2}\right)\right)^{\gamma m} \\ &\xrightarrow[]{m \xrightarrow[]{}\infty} (2k-1)!! e^{\frac{5\gamma k(k-1)}{2}} .
    \end{split}
\end{equation}
Finally, taking $X \sim \mathcal{N}(0,1)$, $Y \sim \mathcal{LN}(-\frac{5}{4}\gamma, \frac{5}{4}\gamma)$ and defining $Z = XY$, we can easily see that 
\begin{equation}
    \mathbb{E}\left[Z^{n}\right]= \begin{cases}0  \hspace{32mm} \text{n odd}\\
   e^{\frac{5}{2}\gamma k(k-1)}(2k-1)!! \hspace{5mm} \text{n = 2k} \end{cases}
\end{equation}

\end{proof}
\section{Additonal results and lemmas}
Here we list some of the moments arising from a Binomial distribution of the form $U \sim \operatorname{Bin}(n, \frac{1}{2})$. We invite the reader to sanity-check our results in Lemma~\ref{binomial-moments} regarding the coefficients of $Q_k$.
\begin{lemma}
Consider the random variable $U \sim \operatorname{Bin}(m, \frac{1}{2})$. We can calculate its first $4$ moments as
\begin{itemize}
    \item $\mathbb{E}[U] = \sum_{i=0}^{m}\binom{m}{i}i = \frac{m}{2}Q_0(m) = \frac{m}{2}$
    \item $\mathbb{E}[U^2] = \sum_{i=0}^{m}\binom{m}{i}i^2 = \frac{m}{2^2}Q_1(m) = \frac{m}{2^2}\left(m + 1\right)$
    \item $\mathbb{E}[U^3] = \sum_{i=0}^{m}\binom{m}{i}i^3 = \frac{m}{2^3}Q_2(m) = \frac{m}{2^3}\left(m^{2} + 3m\right)$
    \item $\mathbb{E}[U^4] = \sum_{i=0}^{m}\binom{m}{i}i^4 = \frac{m}{2^4}Q_3(m) = \frac{m}{2^4}\left(m^{3} + 6m^2 +  3m + 4 \right)$
\end{itemize}
\end{lemma}

\end{document}